\documentclass{article}

\usepackage{bookmark}
\usepackage{natbib}
\usepackage{hyperref}       % hyperlinks
\hypersetup{
    colorlinks=true,
    linkcolor=teal,
    filecolor=teal,      
    urlcolor=teal,
    citecolor = teal
}
\usepackage{url}            % simple URL typesetting
\usepackage{booktabs}       % professional-quality tables
\usepackage{amsfonts}       % blackboard math symbols
\usepackage{nicefrac}       % compact symbols for 1/2, etc.
\usepackage{microtype}      % microtypography
\usepackage{amsthm}
\usepackage{amsmath}
\usepackage{mathrsfs,amssymb}
\usepackage{xcolor}
\usepackage{physics}
\usepackage{enumitem}

\usepackage{framed}
\usepackage[ruled,noline]{algorithm2e}

\newtheorem{theorem}{Theorem}[section]
\newtheorem{lemma}[theorem]{Lemma}
\newtheorem{claim}[theorem]{Claim}
\newtheorem{corollary}[theorem]{Corollary}

\newtheorem{definition}[theorem]{Definition}

\newtheorem{fact}[theorem]{Fact}

\newcommand\numberthis{\addtocounter{equation}{1}\tag{\theequation}}

\newcommand{\vc}{\mathrm{VCDim}}

\newcommand{\br}{\mathbb{R}}

\newcommand{\F}{\mathcal{F}}
\newcommand{\X}{\mathcal{X}}
\newcommand{\Y}{\mathcal{Y}}
\newcommand{\A}{\mathcal{A}}
\newcommand{\D}{\mathcal{D}}

\newcommand{\M}{\mathcal{M}}

\renewcommand{\H}{\mathcal{H}}
\newcommand{\U}{\mathcal{U}}
\newcommand{\G}{\mathcal{G}}
\newcommand{\Q}{\mathcal{Q}}

\newcommand{\nadist}{\mathbfcal{D}}
\newcommand{\adist}{{\pmb{\mathscr{D}}}}
\newcommand{\err}{\mathrm{err}}
\renewcommand{\vec}[1]{\mathbf{#1}}

\newcommand{\bnumber}[3]{\mathcal{N}_{[\,]}\!\left(#1, #2, #3\right)}
\newcommand{\bnet}{\mathcal{B}} %%% this is for a set of brackets. 
\newcommand{\B}{\mathcal{B}} %% This is also for a set of brackets. WE may want to consolidate these if they remain the same. 

\newcommand{\regret}{\textsc{Regret}}

\DeclareMathAlphabet\mathbfcal{OMS}{cmsy}{b}{n}
\usepackage{amsbsy}
\DeclareMathAlphabet{\mathbfscr}{OMS}{mdugm}{b}{n}

\usepackage[left=1in,bottom=1in,top=1in,right=1in]{geometry}

\DeclareMathOperator*{\E}{\mathbb{E}}
\DeclareMathOperator*{\argmin}{\mathrm{argmin}}

\title{\textbf{Smoothed Analysis of Online and Differentially Private Learning}}

\author{{Nika Haghtalab} \thanks{Cornell University; Email: nika@cs.cornell.edu} \and {Tim Roughgarden} \thanks{Columbia University; Email: tr@cs.columbia.edu} \and {Abhishek Shetty} \thanks{Cornell University; Email: shetty@cs.cornell.edu} }

\begin{document}

\maketitle
\begin{abstract}
  Practical and pervasive needs for robustness and privacy in
  algorithms have inspired the design of online adversarial and
  differentially private learning algorithms. The primary quantity
  that characterizes learnability in these settings is the Littlestone
  dimension of the class of
  hypotheses~\citep{PrivatePACLD,ben2009agnostic}.  This
  characterization is often interpreted as an impossibility result
  because classes such as linear thresholds and neural networks have
  infinite Littlestone dimension.  In this paper, we apply the
  framework of smoothed analysis~\citep{ST04}, 
in which adversarially chosen inputs are perturbed slightly by nature. 
We show that fundamentally stronger regret and error
  guarantees are possible with smoothed adversaries than with
  worst-case adversaries. In particular, we obtain regret and privacy
  error bounds that depend only on the VC dimension and the bracketing
  number of a hypothesis class, and on the magnitudes of the
  perturbations.
\end{abstract}

\section{Introduction}
% \vspace*{-4pt}

Robustness to changes in the data and protecting the privacy of data
are two of the main challenges faced by machine learning and have led
to the design of \emph{online} and \emph{differentially private}
learning algorithms.  While offline PAC learnability is
characterized by the finiteness of VC dimension, online and
differentially private learnability are both characterized by the
finiteness of the Littlestone dimension~\citep{PrivatePACLD,ben2009agnostic,bun2020equivalence}.
This latter characterization is often interpreted as an impossibility
result for achieving robustness and privacy on worst-case instances,
especially in classification where even simple hypothesis classes
such as $1$-dimensional thresholds have constant VC dimension but
infinite Littlestone dimension.

Impossibility results for worst-case adversaries do not invalidate the
original goals of robust and private learning with respect to
practically relevant hypothesis classes; rather, they indicate
that a new model is required to provide rigorous guidance on the
design of online and differentially private learning algorithms.
In this work, we go beyond worst-case analysis and \emph{design online
  learning algorithms and differentially private learning algorithms
  as good as their offline and non-private PAC learning counterparts 
in a realistic semi-random model of data.}

Inspired by smoothed analysis~\citep{ST04}, we introduce frameworks
for online and differentially private learning in which adversarially
chosen inputs are perturbed slightly by nature (reflecting, e.g.,
measurement errors or uncertainty).
Equivalently, we consider an adversary restricted to choose an input
distribution that is not overly concentrated, with the realized input
then drawn from the adversary’s chosen distribution.  Our goal is to
design algorithms with good expected regret and error bounds, where
the expectation is over nature’s perturbations (and any random coin
flips of the algorithm).  Our positive results show, in a precise
sense, that the known lower bounds for worst-case online and
differentially private learnability are fundamentally brittle.

\paragraph{Our Model.}
Let us first consider the standard online learning setup with an
instance space $\X$ and a set $\H$ of binary hypotheses each mapping
$\X$ to $\Y =\{+1,-1\}$.  Online learning is played over $T$ time
steps, where at each step the learner picks a prediction function from
a distribution
and the \emph{adaptive} adversary chooses a pair of $(x_t, y_t) \in \X\times \Y$. The regret of an algorithm
is the difference between the number of mistakes the algorithm makes
and that of the best fixed hypothesis in $\H$. The basic goal in
online learning is to obtain a regret of $o(T)$.  In comparison, in
differential privacy the data set
$B =\{(x_1,y_1), \dots, (x_n, y_n) \}$ is specified ahead of time. 
Our goal here is to design a randomized mechanism that with high
probability finds a nearly optimal hypothesis in $\H$ on the set $B$,
while ensuring that the computation is \emph{differentially private.}
 That is, changing a single
element of $B$ does not significantly alter the probability with which
our mechanism selects an outcome. Similar to agnostic PAC learning,
this can be done by ensuring that the error of each hypothesis
$h\in \H$ on $B$ (referred to as a query) is calculated accurately and
privately.

We extend these two models to accommodate smoothed adversaries.  We
say that a distribution $\D$ over instance-label pairs is
{\em $\sigma$-smooth} if its density function over the instance domain is
pointwise bounded by at most $1/\sigma$ times that of the uniform
distribution.  In the online learning setting this means that at step
$t$, the adversary chooses an arbitrary $\sigma$-smooth distribution
$\D_t$ from which $(x_t, y_t)\sim \D_t$ is drawn. In the differential
privacy setting, we work with a database $B$ for which the answers to
the queries could have been produced by a $\sigma$-smooth
distribution.

Why should smoothed analysis help in online learning?  Consider the
well-known lower bound for $1$-dimensional thresholds over
$\X = [0,1]$, in which the learner may as well perform binary search
and the adversary selects an instance within the uncertainty region of
the learner that causes a mistake.  While the learner's uncertainty
region is halved each time step, the worst-case adversary can use
ever-more precision to force the learner to make mistakes
indefinitely. On the other hand, a $\sigma$-smoothed adversary
effectively has bounded precision. That is, once the width of the
uncertainty region drops below $\sigma$, a smoothed adversary can no
longer guarantee that the chosen instance lands in this region.
Similarly for differential privacy, there is a $\sigma$-smooth
distribution that produces the same answers to the queries. Such a
distribution has no more than $\alpha$ probability over an interval of
width $\sigma\alpha$. So one can focus on computing the errors of the
$1/(\sigma\alpha)$ hypotheses with discreized thresholds and learn a
hypothesis of error at most $\alpha$.
Analogous observations have been made in prior works
(\cite{NIPS2011_4262}, \cite{Cohen-Addad},
\cite{Gupta_Roughgarden}), although only for very specific
settings (online learning of $1$-dimensional thresholds,
$1$-dimensional piecewise constant functions, and parameterized greedy
heuristics for the maximum weight independent set problem, respectively).
Our work is the first to demonstrate the breadth of the settings
in which fundamentally stronger learnability guarantees are possible
for smoothed adversaries than for worst-case adversaries.

\paragraph{Our Results and Contributions.} 
\begin{itemize}
  % [topsep=0pt,itemsep=0ex,partopsep=1ex,parsep=1ex,labelindent=0pt,leftmargin=* ]
\item Our main result concerns online learning with \emph{adaptive} $\sigma$-smooth adversaries where $\D_t$ can depend on the history of the play, including the earlier realizations of $x_\tau\sim \D_\tau$ for $\tau < t$. That is, $x_t$ and $x_{t'}$ can be highly correlated. 
We show that regret against these powerful adversaries is bounded by
$\tilde O (\sqrt{T \ln(\mathcal{N})})$, where $\mathcal{N}$ is the
\emph{bracketing number} of $\H$ with respect to the uniform
distribution.\footnote{Along the way, we also demonstrate a stronger
  regret bound for the simpler case of non-adaptive adversaries,
for which each distribution $\D_t$ is independent of the realized
  inputs in previous time steps.}
Bracketing number is the size of an $\epsilon$-cover of $\H$ with the additional property that hypotheses in the cover are \emph{pointwise} approximations of those in $\H$.
We show that for many hypothesis classes, the bracketing number is nicely bounded as a function of the VC dimension. This leads to the regret bound of $\tilde O( \sqrt{T\ \vc(\H)\ln(1/\sigma)})$ for  commonly used hypothesis classes in machine learning, such as halfspaces, polynomial threshold functions, and polytopes. 
In comparison, these hypothesis classes have infinite Littlestone
dimension and thus cannot be learned  with regret $o(T)$ in the worst case~\citep{ben2009agnostic}.

From a technical perspective, we introduce a novel approach for
bounding time-correlated non-independent stochastic processes over
infinite hypothesis classes using the notion of bracketing
number. Furthermore, we introduce systematic approaches, such as
high-dimensional linear embeddings and $k$-fold operations, for
analyzing the bracketing number of complex hypothesis classes. We
believe these techniques are of independent interest.

\item For differentially private learning, we obtain an error bound of 
  $\tilde O\big( \ln^{\frac 38}(1/\sigma)\sqrt{\vc(\H) /n} \big)$; the
  key point is that this bound is independent of the size~$|\X|$ of
  the domain and the size $|\H|$ of the hypothesis class.  We obtain
  these bounds by modifying two commonly used mechanisms in
  differential privacy, the Multiplicative Weight Exponential Mechanism of
  \cite{NIPS2012_4548} and the SmallDB algorithm of \cite{SmallDB}. 
With worst-case adversaries,
these algorithms achieve only error bounds of
  $\tilde O(\ln^{\frac 14}(|\X|)\sqrt{\ln(|\H|) /n})$ and
  $\tilde O( \sqrt[3]{\vc(\H) \ln(|\X|) / n} )$, respectively.
Our results also improve over those in~\citet{DPMultWeights} which
concern a similar notion of
  smoothness and achieve an error bound of
  $\tilde O( \ln^{\frac 12}(1/\sigma)\sqrt{ \ln(|\H|) / n})$.

\end{itemize}

\paragraph{Other Related Works.}
At a higher level,  our work is related to several works on the intersection of machine learning and beyond the worst-case analysis of algorithms (e.g.,  \citep{Dispersion,dekel2017online,kannan2018smoothed})
that are covered in more detail in \hyperref[app:related]{Appendix \ref*{app:related}}.
\section{Preliminaries}
% \vspace{-4pt}
\textbf{Online Learning.}~
We consider a measurable instance space $\X$ and the
label set $\Y = \{+1 , -1\}$. Let $\H$ be a hypothesis class on $\X$ with its  VC dimension denoted by $\vc(\H)$.
Let $\U$ be the uniform distribution over $\X$ with density  function $u(\cdot)$.
For a distribution $\D$ over $\X\times \Y$, let $p(\cdot)$ be the \emph{probability density function} of its marginal over $\X$. We say that $\D$ is \emph{$\sigma$-smooth} if for all $x\in \X$, $p(x) \leq u(x) \sigma^{-1}$.
For a labeled pair $s = (x, y)$ and a hypothesis $h\in \H$,
$\err_{s}(h) = 1(h(x) \neq y)$ indicates whether $h$ makes a
mistake on $s$.

We consider the setting of  \emph{online adversarial and (full-information) learning}. In this setting, a learner and an adversary play a repeated game over $T$ time steps.
In every time step $t\in [T]$ the learner picks a hypothesis $h_t$ and adversary picks a $\sigma$-smoothed distribution $\D_t$ from which a labeled pair $s_t = (x_t, y_t)$ such that $s_t\sim \D_t$ is generated. The learner then incurs penalty of   $\err_{s_t}(h_t)$.
We consider two types of adversaries. First (and the subject of our main results) is called 
 an \emph{adaptive} $\sigma$-smooth adversary. This adversary at every time step  $t\in [T]$ chooses $\D_t$ based on the actions of the
  learner $h_1,\dots, h_{t-1}$ and, importantly, the realizations of the previous
  instances $s_1, \dots,  s_{t-1}$.  We denote this adaptive random process by
  $\vec s \sim \adist$.
A second and less powerful type of adversary is called a  \emph{non-adaptive} $\sigma$-smooth adversary. Such an adversary first chooses an unknown sequence of
  distributions $\nadist = (\D_1, \dots, \D_T)$ such that $\D_t$ is a
  $\sigma$-smooth distribution for all $t\in [T]$. Importantly, $\D_t$  does not depend on realizations of adversary's earlier actions $s_1, \dots, s_{t-1}$ or the learner's actions $h_1, \dots, h_{t-1}$. We denote this non-adaptive random process by $\vec s \sim \nadist$.
 With a slight abuse of notation, we denote by $\vec x\sim \adist$ and $\vec x\sim \nadist$ the sequence of (unlabeled) instances in $\vec s \sim \adist$ and $\vec s\sim \nadist$.

Our goal is to design an online algorithm $\A$ such that expected regret against an adaptive adversary,
\[
\E[\regret(\A, \adist)]{:=}\E_{\vec s\sim \adist}  \left[ \sum_{t=1}^T \err_{s_t}(h_t)
 - \min_{h\in \H}  \sum_{t=1}^T \err_{s_t}(h)
\right]
\]
is sublinear in $T$. We also consider the regret of an algorithm against a non-adaptive adversary defined similarly as above and denoted by  $\E[\regret(\A, \nadist)]$.

\paragraph{Differential Privacy.}
We also consider differential privacy. In this setting, a data set $S$ is a multiset of elements from domain $\X$. Two data sets $S$ and $S'$ are said to be \emph{adjacent} if they differ in at most one element. 
A randomized algorithm $\M$ that takes as input a data set is $(\epsilon, \delta)$-differentially private if for all $\mathcal{R}\subseteq \mathrm{Range(\M)}$ and for all adjacent data sets $S$ and $S'$, $\Pr \left[ \M \left(S \right) \in \mathcal{R} \right] \leq \exp(\epsilon) \Pr\left[  \M\left(S' \right) \in \mathcal{R}  \right] + \delta$. 
    % \begin{equation*}  
    % \end{equation*}
    If $\delta = 0$, the algorithm is said to be purely $\epsilon$-differentially private.

For differentially private learning, one considers a fixed class of \emph{queries} $\Q$. 
The learner's goal is to evaluate these queries on a given data set $S$.
For ease of notation, we work with the empirical distribution $\D_S$ corresponding to a data set $S$. 
Then the learner's goal is to approximately compute $q(\D_S) = \E_{x\sim \D_S}[q(x)]$ while preserving privacy\footnote{In differentially private learning, queries are the error function of hypotheses  and take as input a pair $(x,y)$.}.
We consider two common paradigms of differential privacy. First, called  \emph{query answering}, involves designing a mechanism that outputs  values $v_q$ for all $q\in \mathcal{Q}$ such that
with probability $1-\beta$ for every $q\in Q$, $|q(\D_S) - v_q| \leq \alpha$.
The second paradigm, called \emph{data release}, involves designing a mechanism that outputs a synthetic distribution $\overline{\D}$, such that with probability $1-\beta$ for all $q\in \mathcal{Q}$, $|q(\overline{\D})  - q(\D_S)|\leq \alpha$. That is, the user can use $\overline{\D}$ to compute the value of any $q(\D_S)$ approximately.

Analogous to the definition of smoothness in online learning, we say that a distribution $\D$ with density function $p(\cdot)$ is \emph{$\sigma$-smooth} if $p(x) \leq \sigma^{-1}u(x)$ for all $x \in \X$. We also work with a weaker notion of smoothness  of data sets. A data set $S$ is said to be 
\emph{$\left(\sigma, \chi\right)$-smooth} with respect to a query set $\Q$ if there is a $\sigma$-smooth
 distribution $\D $ such that for all $q \in \Q$, we have $ \abs{ q\left({\D}  \right)  - q\left(\D_S \right) } \leq \chi   $. 
The definition of $\left(\sigma, \chi\right)$-smoothness, which is also referred to as \emph{pseudo-smoothness} by 
\cite{DPMultWeights}, captures data sets that though might be concentrated on some elements, the query class is not capable of noticing their lack of smoothness.

\paragraph{Additional Definitions.}
Let $\H$ be a hypothesis class and let $\D$ be a distribution.
$\H'$ is an $\epsilon$-cover for $\H$ under $\D$ if for all $ h\in \H $, there is a $h' \in \H' $ such that $ \Pr_{x \sim \D}  \left[ h \left( x \right) \neq h'\left( x \right)  \right]  {\leq \epsilon}  $.
For any $\H$ and $\D$, there an $\epsilon$-cover $\H'\subseteq \H$ under  $\D$ such that  $|\H'| \leq (41/\epsilon)^{\vc(\H)}$ (\cite{HAUSSLER1995217}).

We define a partial order $\preceq$ over functions such that $f_1\preceq f_2$ if and only if for all $x\in \X$, we have $f_1(x) \leq f_2(x)$.
For a pair of functions $f_1, f_2$ such that $f_1\preceq f_2$, a \emph{bracket} $\left[ f_1 , f_2 \right]$ is defined by $
\left[ f_1 , f_2 \right] = \left\{ f : \X \to \left\{ -1,1\right\} : f_1 \preceq f \preceq f_2 \right\}.$
Given a measure $\mu$ over $\X$, a bracket $\left[f_1, f_2 \right]$ is called an \emph{$\epsilon$-}bracket if $\Pr_{x\sim \mu}\left[ f_1(x) \neq f_2(x) \right]\leq \epsilon$. 

\begin{definition}[Bracketing Number]
Consider an instance space $\X$, measure $\mu$ over this space, and hypothesis class $\F$.
A set $\bnet$  of brackets is called an \emph{$\epsilon$-bracketing of $\F$ with respect to measure $\mu$} if all brackets in $\bnet$ are $\epsilon$-brackets with respect to $\mu$ and for  every $f\in \F$ there is $[f_1, f_2] \in \bnet$ such that $f\in [f_1,f_2]$.
The \emph{$\epsilon$-bracketing number} of $\F$ with respect to measure $\mu$, denoted by  $\bnumber{\mathcal{F}}{ \mu }{ \epsilon }$, is the size of the smallest $\epsilon$-bracketing for $\F$ with respect to $\mu$.
\end{definition}
\section{Regret Bounds for Smoothed Adaptive and Non-Adaptive Adversaries}
\label{sec:regret-prelim}
In this section, we obtain regret bounds against smoothed adversaries.
For finite hypothesis classes $\H$, existing no-regret algorithms such as Hedge~\citep{HEDGE} and Follow-the-Perturbed-Leader~\citep{FTPL} achieve a regret bound of $O( \sqrt{T \ln(\H)} )$.
For a possibly infinite hypothesis class our approach uses a finite set $\H'$ as a \emph{proxy} for $\H$ and only focuses on competing with hypotheses in $\H'$ by running a standard no-regret algorithm on $\H'$.
Indeed, in absence of smoothness of $\adist$, $\H'$ has to be a good  proxy with respect to every distribution or know the adversarial sequence ahead of time, neither of which are possible in the online setting. But when distributions are smooth, $\H'$ that is a good proxy for the uniform distribution can also be a good proxy for all other smooth distributions.
We will see that  how well a set $\H'$ approximates $\H$ depends on adaptivity (versus non-adpativity) of the adversary. 
Our main technical result in \hyperref[sec:regret]{Section \ref*{sec:regret}} shows that
 for adaptive adversaries this approximation depends on the size of the $\frac{\sigma}{4\sqrt{T}}$-bracketing cover of $\H$. This results in an algorithm whose regret 
is sublinear in $T$ and logarithmic in that bracketing number for adaptive adversaries (\hyperref[thm:main-ada]{Theorem~\ref*{thm:main-ada}}). In comparison,  for simpler non-adaptive adversaries this approximation depends on the size of the more traditional $\epsilon$-covers of $\H$, which do not require pointwise approximation of $\H$.   This leads to an algorithm against non-adaptive adversaries with an improved regret bound of  $\tilde O(\sqrt{T \cdot  \vc(\H)})$ (\autoref{thm:main-ada}). 

In \hyperref[sec:bracket-eg]{Section \ref*{sec:bracket-eg}}, we demonstrate that the bracketing numbers of commonly used hypothesis classes in machine learning are small functions of their VC dimension. We also provide systematic approaches for bounding the bracketing number of complex hypothesis classes in terms of the bracketing number of their simpler  building blocks. This shows that for many commonly used hypothesis classes --- such as halfspaces, polynomial threshold functions, and polytopes --- we can achieve a regret of $\tilde O(\sqrt{T \cdot  \vc(\H)})$ even against an adaptive adversary.

\subsection{Regret Analysis and the Connection to Bracketing Number}
\label{sec:regret}
In more detail,  consider an algorithm $\A$ that uses Hedge on a finite set $\H'$ instead of $\H$. Then,
    \begin{equation} \label{eq:regretdecomposition}
        \E[\regret(\A, \adist)] \leq  O\left( \sqrt{T \ln(|\H'|)} \right) + \E_{\adist}\left[\max_{h\in \H}\min_{h'\in \H'} \sum_{t=1}^T 1\left(h(x_t) \neq h'(x_t) \right)  \right],
    \end{equation}
where the first term is the regret against the best $h'\in \H'$ and the second term captures how well $\H'$ approximates $\H$.
A natural choice of $\H'$ is an $\epsilon$-cover of $\H$ with respect to the uniform distribution, for a small $\epsilon$ that will be defined later.
This bounds the first term using the fact that there is an $\epsilon$-cover $\H'\subseteq \H$ of size $|\H'|\leq (41/\epsilon)^{\vc(\H)}$.
To bound the second term, we need to understand whether there is a hypothesis $h\in \H$ whose value over \emph{an adaptive sequence  of $\sigma$-smooth distributions} can be drastically different from the value of its closest (under uniform distribution) proxy $h'\in \H'$. Considering the symmetric difference functions $f_{h, h'} = h\Delta h'$ for functions $h\in \H$ and their corresponding proxies $h'\in \H'$, we need to bound (in expectation) the maximum value an $f_{h,h'}$ can attain over an adaptive sequence of $\sigma$-smooth distributions.

\paragraph{Non-Adaptive Adversaries.} To develop more insight, let us first consider the case of \emph{non-adaptive} adversaries. In the case of non-adaptive adversaries,  $x_t\sim \D_t$ are \emph{independent} of each other, while they are not identically distributed.
This independence is the key property that allows us to use the VC dimension of the set of functions $\{f_{h, h'}\mid \forall h\in \H \text{ and the corresponding proxy } h'\in \H' \}$ to establish a uniform convergence property where with high probability every function $f_{h,h'}$ has a value that is close to its expectation --- the fact that $x_t$s are not identically distributed can be easily handled because the double sampling and symmetrization trick in VC theory can still be applied as before.
Furthermore, $\sigma$-smoothness of the distributions implies that 
$\E_{\nadist}[\sum f_{h,h'}(x_t)] \leq \sigma^{-1} \E_{\U}[\sum f_{h,h'}(x_t)] \leq \epsilon/\sigma$. This leads to the following theorem for non-adaptive adversaries.

\begin{theorem}[Non-Adaptive Adversary~\citep{haghtalab2018foundation}] \label{thm:nonadaptive}
Let $\H$ be a hypothesis class of VC dimension $d$. There is an algorithm such that for any  $\nadist$ that is an non-adaptive sequence of $\sigma$-smooth distributions has regret
$    \E[\regret(\A,\nadist)] \in O \left( \sqrt{ d T \ln\left(\frac T \sigma \right)} \right).$
\end{theorem}
\paragraph{Adaptive Adversaries.}
Moving back to the case of adaptive adversaries, we unfortunately lose this uniform convergence property (see \hyperref[app:ExampleUC]{Appendix \ref*{app:ExampleUC}} for an example).
This is due to the fact that now the choice of $\D_t$ can depend on the earlier realization of instances $x_1, \dots, x_{t-1}$.
To see why independence is essential, note that the ubiquitous double sampling and symmetrization techniques used in VC theory require that taking two sets of samples $\vec x$ and $\vec x'$ from the process that is generating data, we can swap $x_i$ and $x'_i$ independently of whether $x_j$ and $x'_j$ are swapped for $j\neq i$. When the choice of $\D_t$ depends on $x_1, \dots, x_{t-1}$ then swapping $x_\tau$ with $x'_\tau$ affects whether $x_t$ and $x'_t$ could even be generated from $\D_t$ for $t>\tau$. 
In other words, symmetrizing the first $t$ variables generates $2^t$ possible choices for $x^{t+1}$ that exponentially increases the set of samples over which a VC class has to be projected, therefore losing the typical $\sqrt{T \cdot\vc(\H)}$ regret bound and instead obtaining the trivial regret of $O(T)$. Nevertheless, we show that the earlier ideas for bounding the second term of \autoref{eq:regretdecomposition} are  still relevant as long as we can side step the need for independence. 

Note that $\sigma$-smoothness of the distributions still implies that for a fixed function $f_{h,h'}$ even though $\D_t$ is dependent on the realizations $x_1, \dots, x_{t-1}$, we still have $\Pr_{x_t\sim \D_t}[f_{h,h'}(x_t)]\leq \epsilon/\sigma$.
Indeed,  the value of any function $f$ for which $\E_{\U}[f(x)] \leq \epsilon$ can be bounded by the convergence property of an appropriately chosen Bernoulli variable. As we demonstrate in the following lemma, this allows us to bound the expected maximum value of a $f_{h,h'}$ \emph{chosen from a finite set of symmetric differences.}  For a proof of this lemma refer to \hyperref[app:Proofoffinite]{Appendix \ref*{app:Proofoffinite}}.

\begin{lemma}
\label{lem:union-ada}
Let $\F: \X\rightarrow \{0,1\}$ be any finite class of functions such that $\E_{\U}[f(x)]\leq \epsilon$ for all $f\in \F$, i.e., every function has measure $\epsilon$ over the uniform distribution.
Let $\adist$ be any adaptive sequence of $T$, $\sigma$-smooth distributions for some $\sigma\geq \epsilon$ such that $T\frac\epsilon\sigma \geq \sqrt{\ln(|\F|)}$. We have that
\[ 
\E_{\vec x \sim \adist} \left[\max_{f\in \F}\sum_{t=1}^T f(x_t) \right] \leq O\left(T\frac{\epsilon}{\sigma} \sqrt{ \ln(|\F|)} \right).
\]
\end{lemma}

The set of symmetric differences $\G = \{f_{h, h'}\mid \forall h\in \H \text{ and the corresponding proxy } h'\in \H' \}$ we work with is of course infinitely large.  Therefore, to apply \hyperref[lem:union-ada]{Lemma \ref*{lem:union-ada}} we have to approximate $\G$ with a finite set $\F$ such that 
\begin{equation}\label{eq:need-bracket}
\E_{\vec x \sim \adist}\left[\max_{f_{h,h'}\in \G} \sum_{t=1}^T f_{h,h'}(x_t) \right] \lesssim \E_{\vec x \sim\adist}\left[\max_{f\in \F} \sum_{t=1}^T f(x_t)  \right].
\end{equation}
What should this set $\F$ be? Note that choosing $\F$ that is an $\epsilon$-cover of $\G$ under the uniform distribution is an ineffective attempt plagued by the the same lack of independence that we are trying to side step.
In fact, while all functions $f_{h,h'}$ are $\epsilon$ close to the constant $0$ functions with respect to the uniform distribution, they are activated on different parts of the domain. So it is not clear that an adaptive adversary, who can see the earlier realizations of instances, cannot ensure that one of these regions will receive a large number realized instances.
But a second look at \autoref{eq:need-bracket} suffices to see that this is precisely what we can obtain if $\F$ were to be the set of (upper) functions in an $\epsilon$-bracketing of $\G$.
That is,  for every function $f_{h ,h'}\in \G$ there is a function $f \in \F$ such that $f_{h,h'}\preceq f$. This proves  \autoref{eq:need-bracket} with an exact inequality using the fact that  pointwise approximation $f_{h,h'}\preceq f$ implies that the value of $f_{h,h'}$ is bounded by that of $f$ for any set of instances $x_1, \dots, x_T$ that could be generated by $\adist$. Furthermore, functions in $\G$ are 
within $\epsilon$ of the constant $0$ function over the uniform distribution, so $\F$  meets the criteria of \hyperref[lem:union-ada]{Lemma \ref*{lem:union-ada}} with the property that for all $f\in \F$, $\E_\U[f(x)]\leq \epsilon$.
It remains to bound the size of class $|\F|$ in terms of the bracketing number of $\H$. This can be done by showing that the bracketing number of class $\G$, that is the class of all  symmetric differences in $\H$, is approximately bounded by the same bracketing number of $\H$ (See \autoref{thm:k-fold} for more details).
Putting these all together we get the following regret bound against smoothed adaptive adversaries.

\begin{theorem}[Adaptive Adversary]
\label{thm:main-ada}
Let $\H$ be a hypothesis class over domain $\X$, whose $\epsilon$-bracketing number with respect to the uniform distribution over $\X$ is denoted by $\bnumber{\H}{\U}{\epsilon}$. There is an algorithm such that for any  $\adist$ that is an adaptive sequence of $\sigma$-smooth distributions has regret
\[    \E[\regret(\A,\adist)] \in
	 O\left( \sqrt{T \ln\left( \bnumber{\H}{\U}{ \frac{\sigma}{4\sqrt T}} \right)} 
	 \right).
\]
\end{theorem}

\subsection{Hypothesis Classes with Small Bracketing Numbers.}

\label{sec:bracket-eg}

In this section, we analyze bracketing numbers of some commonly used hypothesis classes in machine learning. 
We start by reviewing the bracketing number of halfspaces and provide two systematic approaches for extending this bound to other commonly used hypothesis classes.
Our first approach bounds the bracketing number of any class using the dimension of the space needed to embed it as halfspaces. 
Our second approach shows that $k$-fold operations on any hypothesis class, such as taking the class of intersections or unions of all $k$ hypotheses in a class, only mildly increase the bracketing number. 
Combining these two techniques allows us to bound the bracketing number of commonly used classifiers such as halfspaces, polytopes, polynomial threshold functions, etc.

The connection between bracketing number and VC theory has been explored in recent works. \citet{adams2010,adams2012} showed that  finite VC dimension class also have finite $\epsilon$-bracketing number but \cite{AlonWezlHaussler} (see \cite{UGC} for a modern presentation) showed the dependence on $1/\epsilon$ can be arbitrarily bad.
Since \autoref{thm:main-ada} depends on the growth rate of bracketing numbers, we work with classes for which we can obtain $\epsilon$-bracketing numbers with reasonable growth rate, those that are close to the size of  standard $\epsilon$-covers. 

  \begin{theorem}[\cite{Bracket_Halfspaces}]\label{thm:half_brack}
  Let $\H$ be the class of halfspaces over $\br^d$. For any $\epsilon >0$ and any measure $\mu$ over $\br^d$,
$           \bnumber{\H}{\mu}{\epsilon} \leq \left( \frac{d}{\epsilon} \right)^{ O\left(   d \right) }.$
     \end{theorem}
Our first technique uses this property of halfspaces to bound the bracketing number of any hypothesis class as a function of the dimension of the spaces needed to embed this class as halfpsaces. 

    \begin{definition}[Embeddable Classes]
        Let $\mathcal{G}$ be a hypothesis class on $ \mathcal{X} $. 
        We say that $ \mathcal{G} $ is embeddable as halfspaces in $m$ dimensions if there exists a map $ \psi : \mathcal{X} \to \br^m $ such that for any $g \in \mathcal{G} $, there is a linear threshold function $h$ such $ g = h \circ \psi $.  
    \end{definition}

    \begin{theorem}[Bracketing Number of Embeddable Classes]
    \label{thm:embeddable}
        Let $\mathcal{G}$ be a hypothesis class embeddable as halfspaces in $m$ dimensions. 
        Then, for any measure $\nu$,
$            \bnumber{ \mathcal{G} }{\nu}{\epsilon} \leq \left(  \frac{m}{\epsilon} \right) ^{O\left( m \right)}. 
$
    \end{theorem}    
Our second technique shows that combining $k$ classes, by respectively taking intersections or unions of any $k$ functions from them, only mildly increases their bracketing number.
\begin{theorem}[Bracketing Number of $k$-fold Operations]
\label{thm:k-fold}
    Let $\F_1, \dots, \F_k$ be $k$ hypothesis classes. Let $\F_1\cdot \F_2\cdots \F_k$ and $\F_1+\F_2 + \cdots + \F_k$ be the class of all hypotheses that are  intersections and  unions of $k$ functions $f_i\in \F_i$, respectively. Then, 
\[  \bnumber{\F_1\cdot \F_2 \cdots \F_k}{\mu}{k \epsilon} \leq \prod_{i\in [k]}     \bnumber{\F_i}{\mu}{\epsilon}  \]
and
\[  \bnumber{\F_1+\F_2 +\cdots + \F_k}{\mu}{k\epsilon} \leq \prod_{i\in [k]}  \bnumber{\F_i}{\mu}{\epsilon}. \]
%\ascomment{editted to $k\epsilon$}
For any hypothesis class $\F$ and $\G = \{ f\Delta f' \mid \text{ for all} f, f'\in \F\}$,
$  \bnumber{\G}{\mu}{4 \epsilon} \leq \left( \bnumber{\F}{\mu}{\epsilon} \right)^4.$
\end{theorem}

We now use our techniques for bounding the bracketing number of complex classes by the bracketing number of their simpler building blocks to show that online learning with an adaptive adversary on a class of halfspaces, polytopes, and polynomial threshold functions has $\tilde O(\sqrt{T\ \vc(\H)})$ regret.

\begin{corollary}
\label{cor:bnumber}
Consider  instance space $\X = \br^n$ and let $\mu$ be an arbitrary measure on $\X$. 
Let $\mathcal{P}^{n,d}$ be the class of $d$-degree \emph{polynomial thresholds} and $\mathcal{Q}^{n,k}$ be the class $k$-polytopes in $\br^n$. Then,
\[
\bnumber{\mathcal{P}^{n,d}}{\mu}{\epsilon}  \leq \exp\left(c_1 n^d \ln\left( n^d/\epsilon \right) \right)
 \text{ and }
\bnumber{\mathcal{Q}^{n, k}}{\mu}{\epsilon}\leq   \exp\left(c_2 nk \ln\left( \frac{nk}{\epsilon} \right)  \right),
\]
for some constants $c_1$ and $c_2$. Furthermore, there is an online algorithm whose regret against an adaptive $\sigma$-smoothed adversary on the class 
$\mathcal{P}^{n,d}$
and $\mathcal{Q}^{n, k}$ is respectively $\tilde O(\sqrt{T \cdot \vc(\mathcal{P}^{n,d})\ln(1/\sigma)} )$ and $\tilde O(\sqrt{T \cdot \vc(\mathcal{Q}^{n,k}) \ln(1/\sigma)} )$.
\end{corollary}

\section{Differential Privacy} \label{sec:DP}
In this section, we consider smoothed analysis of differentially private learning in \emph{query answering} and \emph{data release} paradigms.
We primarily focus on $(\sigma , 0)$-smooth distributions and defer the general case of $(\sigma , \chi)$-smooth distributions to \hyperref[app:SmallDB]{Appendix \ref*{app:SmallDB}}.
For finite query classes $\Q$ and small domains, existing differentially private mechanisms achieve an error bound that depends on $\ln(|\Q|)$ and $\ln(|\X|)$. 
We leverage smoothness of data sets to improve these dependencies to $\vc(\Q)$ and $\ln(1/\sigma)$. 

\paragraph{An Existing Algorithm.}
\cite{NIPS2012_4548} introduced a practical algorithm  for data release, called Multiplicative Weights Exponential Mechanism (MWEM). This algorithm works for a finite query class $\Q$ over a finite domain $\X$.
Given an data set $B$ and its corresponding empirical distribution $\D_B$, MWEM iteratively builds distributions $\D_t$ for $t\in[T]$, starting from $\D_1= \U$ that is the uniform distribution over $\X$. 
At stage $t$, the algorithm picks a $q_t \in \Q$ that approximately maximizes the error $\left| q_t\left(\D_{t-1}\right)-q_t(\D_B)\right|$ using a differentially private mechanism (Exponential mechanism). 
Then data set $\D_{t-1}$ is updated using the multiplicative weights update rule 
$\D_t(x) \propto \D_{t-1}(x)  \exp\left( q_t( x) (m_t - q_t(\D_{t-1}))/2 \right)$ 
where $m_t$ is a differentially private estimate (via Laplace mechanism) for the value $q_t \left( \D_B \right)$.  
The output of the mechanism is a data set $ {\overline{\D}} =  \frac 1T \sum_{t\in[T]}\D_t$.
The formal guarantees of the algorithm are as follows.

\begin{theorem}[\cite{NIPS2012_4548}] \label{thm:MWEM}
    For any data set $B$ of size $n$, a finite query class $\Q$, $T \in \mathbb{N}$ and $\epsilon>0$, MWEM is $\epsilon$-differentially private and with probability at least $1 - \nicefrac{2T}{|\mathcal{\Q}|}$ produces a distribution $ {\overline{\D}}$ over $\X$ such that $ \max_{q \in \Q }\left\{   \abs{q\left( {\overline{\D}}\right) - q\left(\D_B \right) } \right\} \leq 2  \sqrt{\frac{\log |\mathcal{X}|}{T}}+\frac{10 T \log |\mathcal{Q}|}{\epsilon n} .$
    % \begin{equation*}
    % \end{equation*}
\end{theorem}
The analysis of MWEM keeps track of the KL divergence $\mathrm{D}_{\mathrm{KL}} \left(\D_B\| \D_t \right)$ and shows that at time $t$ this value decreases by approximately the error of query $q_t$. At a high level,
$\mathrm{D}_{\mathrm{KL}} \left(\D_B\| \D_1 \right) \leq \ln(|\X|)$. Moreover, KL divergence of any two distributions is non-negative. Therefore, error of any query $q\in \Q$ after $T$ steps follows the above bound. 

\paragraph{Query Answering.}
To design a private query answering algorithm
for a query class $\Q$ without direct dependence on $\ln(|\Q|)$ and $\ln(|\X|)$ we leverage smoothness of distributions.
Our algorithm called the Smooth Multiplicative Weight Exponential Mechanism (Smooth MWEM), given an infinite set of queries $ \mathcal{Q} $, considers a $\gamma$-cover $ \mathcal{Q}'$ under the uniform distribution.
Then, it runs the MWEM algorithm with $\mathcal{Q}'$ as the query set and constructs an empirical distribution $ {\overline{\D}}$.
Finally, upon being requested an answer to a query $q \in \mathcal{Q} $, it responds with $ q'(  {\overline{\D}} )$, where $q' \in \mathcal{Q} '$ is the closest query to $q$ under the uniform distribution. 
%  evaluated on the database $A =  1 / T \sum A_i  $.
This algorithm is presented in \hyperref[sec:DPProofs]{Appendix \ref*{sec:DPProofs}}.
Note that $\mathcal{Q}'$ does not depend on the data set $B$.
This is the key property that enables us to work with a finite $\gamma$-cover of $\Q$ and extend the privacy guarantees of  MWEM to infinite query classes. In comparison, constructing a $\gamma$-cover of $\Q$ with respect to the empirical distribution $\D_B$ uses private information.

Let us now analyze the error of our algorithm and outline the reasons it does not directly depend on $\ln(|\Q|)$ and $\ln(|\X|)$.
Recall that from the $\left( \sigma , 0 \right) $-smoothness,  there is a distribution $\overline{\D_B}$ that is $\sigma$-smooth and $q \left( \D_B  \right) = q\left( \overline{\D_B} \right) $ for all $q \in \mathcal{Q} $. 
Furthermore, $\mathcal{Q}'$ can be taken to be a subset of $\mathcal{Q}$ and thus $B$ is $ \left( \mathcal{\sigma} , 0 \right) $-smooth with respect to $\mathcal{Q}' $. 
The approximation of $\mathcal{Q}$ by a $\gamma$-cover introduces error in addition to the error of \autoref{thm:MWEM}. 
This error is given by $ \abs{ q \left( \D_B \right) - q' \left( \D_B \right) } \leq 2 \cdot\Pr_{\U} \left[ q'\left( x \right) \neq q\left( x \right)  \right] \sigma^{-1}  \leq 2\gamma/\sigma$. 
Note that $|\Q'| \leq (41/\gamma)^{\vc( \mathcal{Q} )}$, therefore, this removes the error dependence  on the size of the query set $\mathcal{Q}$ while adding a small error of $2\gamma/\sigma$. 
Furthermore, \autoref{thm:MWEM} dependence on $\ln(|\X|)$ is due to the fact that for a worst-case (non-smooth) data set $B$, $\mathrm{D}_\mathrm{KL}(\D_B\| \U)$ can be as high as $\ln(\abs{ \X }  )$. For a $\left( \sigma , 0 \right)$-smooth data set, however, $\mathrm{D}_\mathrm{KL}( \overline{\D_B} \|\U)\leq \ln(1/\sigma)$. 
This allows for faster error convergence. 
Applying these ideas together and setting $\gamma = \sigma/2n$ gives us the following theorem whose proof is
 {deferred to \hyperref[app:datareleaseanalysis]{Appendix \ref*{sec:DPProofs}}.}

\begin{theorem} \label{thm:DPMW}
    For any $  \left( \sigma , 0 \right) $-smooth dataset $B$ of size $n$, a query class $\Q$ with VC dimension $d$, $T \in \mathbb{N}$ and $\epsilon>0$, \hyperref[alg:QueryAnswering]{Smooth Multiplicative Weights Exponential Mechanism} is $\epsilon$-differentially private and with probability at least $1 - 2T  \left(  \nicefrac{\gamma}{41} \right)^{ \vc\left( \mathcal{Q} \right) }$, 
     {calculates values  $v_{q}$  for all $q\in \Q$ such that}
    \begin{equation*}
        \max_{q \in \Q }\left\{  \abs{ v_{q} - q\left(\D_B \right) } \right\} \leq \frac{1}{n} + 2  \sqrt{\frac{\log\left( \nicefrac{1}{\sigma} \right) }{T}}+\frac{10 T d \log\left( \nicefrac{2n}{\sigma} \right) }{\epsilon n }.
    \end{equation*}
\end{theorem}

\paragraph{Data Release.}
Above we described a procedure for query answering that relied on the construction of a data set. One could ask whether this leads to a solution to the data release problem as well.
An immediate, but ineffective, idea is to output distribution $ {\overline{\D}}$ constructed by our algorithm in the previous section. 
The problem with this approach is that while $q'( {\overline{\D}})\approx q'(\D_B)$ for all queries in the cover $\Q'$, there can be queries $q\in \Q\setminus \Q'$ for which $\abs{q( {\overline{\D}}) - q(\D_B)}$ is quite large. This is due to the fact that even though  {$B$ is $(\sigma,0)$-smooth (and $\overline{\D_B}$ is $\sigma$-smooth),} the repeated application of multiplicative update rule may result in distribution $ {\overline{\D}}$ that is far from being smooth.

To address this challenge, we introduce
\hyperref[alg:QueryAnswering]{Projected Smooth Multiplicative Weight Exponential Mechanism} (Projected Smooth MWEM)
that
ensures that $\D_t$ is also $\sigma$-smooth by projecting it on the convex set of all $\sigma$-smooth distributions. More formally, let $\mathcal{K}$ be the polytope of all $\sigma$-smooth distributions over $\X$ and let $\tilde{\D}_t$ be the outcome of the multiplicative update rule of \cite{NIPS2012_4548} at time $t$. Then, Projected Smooth MWEM mechanism uses 
$\D_t=  \argmin_{\D \in \mathcal{K} } \mathrm{D}_{\mathrm{KL}}(\, \D \| \tilde{\D }_t  ).$
To ensure that these projections do not negate the progress made so far, measured by the decrease in KL divergence, we note that for any $ \overline{ \D_B} \in \mathcal{K}$ and any $\tilde \D_t$, we have
$\mathrm{D}_{\mathrm{KL}}( \overline{ \D_B} \| \tilde\D_t\, ) \geq \mathrm{D}_{\mathrm{KL}}( \overline{ \D_B} \| \D_t)+\mathrm{D}_{\mathrm{KL}}( \D_t \| \tilde\D_t ).$
That is, as measured by the decrease in KL divergence, the improvement with respect to $\D_t$ can only be greater  than that of $\tilde \D_t$.
Optimizing parameters $T$ and $\gamma$, we obtain the following guarantees. 
See  \hyperref[sec:DPDataRelease]{Appendix \ref*{sec:DPDataRelease}} for more details on Projected Smooth MWEM mechanism and its analysis.
\begin{theorem}[Smooth Data Release]
\label{thm:PSMWEM}
    Let $B$ be a $\sigma$-smooth database with $n$ data points. 
    For any $\epsilon , \delta > 0 $ and any query set $\mathcal{Q}  $ with VC dimension $d$,  \hyperref[alg:DataRelease]{Projected Smooth Multiplicative Weight Exponential Mechanism} is $\left( \epsilon , \delta \right)$ differentially private and with probability at least $1 - 1/poly\left( n / \sigma \right)^d$ its outcome $ {\overline{\D}}$ satisfies
       \begin{equation*}
       \max_{q\in \mathcal{Q}} \left\{   \abs{q\left(\, {\overline{\D}} \right) - q\left( \D_B \right) }    \right\} \leq O\left( \sqrt{ \frac{d}{\epsilon n}    \log ^{ \frac{1}{2} }  \left(\frac{1}{\sigma} \right)  \log\left( \frac{n}{\sigma } \right) \log\left( \frac{1}{ \delta}   \right)} \right). 
    \end{equation*}
\end{theorem}

\section{Conclusions and Open Problems}
Our work introduces a framework for smoothed analysis of online and private learning
and obtain regret and error bounds that depend only on the VC dimension and the bracketing
  number of a hypothesis class and are independent of the domain size and Littlestone dimension.

Our work leads to several interesting questions for future work.
The first is to characterize learnability in the smoothed setting --- via matching lower bounds ---  in terms of  a combinatorial quantity, e.g., bracketing number.
In \hyperref[app:brackextra]{Appendix \ref*{app:brackextra}}, we discuss \emph{sign rank} and its connection to bracketing number
as a promising candidate for this characterization.
A related question is whether  there are finite VC dimension classes that cannot be learned in presence of smoothed adaptive adversaries.

Let us end this paper by noting that the Littlestone dimension plays a key role in characterizing learnability and algorithm design in the worst-case for several socially and practically important constraints~\citep{ben2009agnostic,AlonWezlHaussler}.
It is essential then to develop models that can bypass Littlestone impossibility results and  provide  rigorous guidance in achieving these constraints in practical settings.

\section*{Acknowledgements}

This work was partially supported by the NSF under CCF-1813188, the ARO under W911NF1910294 and a JP Morgan Chase Faculty Fellowship. 

     \bibliographystyle{plainnat}
    \bibliography{Ref}

\appendix

\section{Additional Related Work} \label{app:related}

Analogous models of smoothed online learning have been explored in prior work. 
\citet{NIPS2011_4262} consider online learning when the adversary is constrained in several ways and work with a notion of sequential Rademacher complexity for analyzing the regret.
 In particular, they study a related notion of  smoothed adversary and show that one can learn thresholds with regret of $O( \sqrt{T} )$ in presence of smoothed adversaries. 
\citet{Gupta_Roughgarden} consider smoothed online learning in the context online algorithm design.  
They show that while optimizing parameterized greedy heuristics  for Maximum Weight Independent Set imposes linear regret in the worst-case, in presence of smoothing this problem can be learned with sublinear regret (as long they allow per-step runtime that grows with $T$).
\cite{Cohen-Addad} consider the same problem with an emphasis on the per-step runtime  being logarithmic in $T$. They show that piecewise constant functions over the interval $[ 0,1 ]$ can be learned efficiently within regret of $O ( \sqrt{T} ) $ against a \emph{non-adaptive} smooth adversary.
Our work differs from these by upper bounding the regret using a combinatorial dimension of the hypothesis class and demonstrating techniques that generalize to large class of problems in presence of \emph{adaptive} adversaries.

{In another related work, \cite{Dispersion} introduce a notion of dispersion in online optimization (where the learner picks an instance and the adversary picks a function) that is a constraint on the number of discontinuities in the adversarial sequence of functions. They show that online optimization can be done efficiently under certain assumptions. Moreover, they show that sequences generated by \emph{non-adaptive} smooth adversaries in one dimension satisfy dispersion.
In comparison, our main results in online learning consider the more powerful adaptive adversaries.
}

{
Smoothed analysis is also used in a number of other online settings. In the setting of linear contextual bandits, \cite{kannan2018smoothed} use smoothed analysis to show that the greedy algorithm achieves sublinear regret even though in the worst case it can have linear regret. \cite{raghavan2018externalities} work in a Bayesian version of the same setting and achieve improved regret bounds for the greedy algorithm. 
{Since several algorithms are known to have sublinear regret in the linear contextual bandit setting even in the worst-case, the main contribution of these papers is to show that the simple and practical greedy algorithm has much better regret guarantees than in the worst-case.}
In comparison, we work with a setting where no algorithm can achieve sublinear regret in the worst-case.
}

{Smoothed analysis has also been  considered in the context of differential privacy. 
\cite{DPMultWeights} consider differential privacy in the interactive setting, where the queries arrive online. 
They analyze a multiplicative weights based algorithm whose running time and error they show can be vastly improved in the presence of smoothness. 
 {Some of our techniques for query answering and data release are inspired by that line of work.}
\cite{Dispersion} also differential privacy in presence of dispersion and analyze the gaurantees of the exponential mechanism.}

{Generally, our work is also related to a line of work on online learning in presence of additional assumptions resembling properties exhibited by real life data. 
\cite{PredictableSequences} consider settings where additional information in terms of an estimator for  future instances is available to the learner. They achieve regret bounds that are in terms of the path length of these estimators and can  beat $\Omega(\sqrt{T})$ if the estimators are accurate. 
\cite{dekel2017online} also considers the importance of incorporating side information in the online learning framework and show that regrets of $O(\log(T))$ in online linear optimization maybe possible when  the learner knows a vector that is weakly correlated with the future instances.}

{
More broadly, our work is among a growing line of work on beyond the worst-case analysis of algorithms~\citep{roughgarden_2020}
that 
considers the design and analysis of algorithms on instances that satisfy properties demonstrated by  real-world applications. Examples of this in theoretical machine learning mostly include improved runtime and approximation guarantees of numerous supervised
(e.g., \citep{LearningSmoothed,kalai2008decision,onebit,Masart}),
and unsupervised settings~(e.g., \citep{bilu_linial_2012, kcenter, stable_clustering, TopicModelling, Decoupling,VDW, MMVMaxCut,llyods, HardtRoth}). 
}

\section{Lack of Uniform Convergence with Adaptive Adversaries} \label{app:ExampleUC}
The following example for showing lack of uniform convergence over adaptive sequences  is due to ~\citet{haghtalab2018foundation} and is included here for completeness.

    Let $\X = [0,1]$ and $\mathcal{G} = \{g_b(x) = \mathbb{I}(x\geq b) \mid \forall b\in[0,1] \}$ be the set of one-dimensional thresholds.
    Let the distribution of the noise $\eta_i$ be the uniform distribution on $\left( - \nicefrac{1}{4} , \nicefrac{1}{4} \right)$. 
    Let $x_1 = \nicefrac{1}{2} $ and $x_2 = x_3 = \dots = x_T = \nicefrac{1}{4} $ if $\eta_1 \leq 0 $ while $x_2 = x_3 = \dots = x_T = \nicefrac{3}{4} $ otherwise. 
     In this case, we do not achieve concentration for any value of $T$, as 
    \[
    \frac 1T \sum_{t=1}^T g_{0.5}(x_t + \eta_t) = \begin{cases} 0 & \text{w.p. $ \nicefrac{1}{2} $}\\ 1 & \text{w.p. $ \nicefrac{1}{2} $} \end{cases} \qquad and \qquad  \mathbb{E} \left[  \frac 1T \sum_{t=1}^T g_{0.5}(x_t + \eta_t) \right] = \frac 12.
    \]

\section{Proofs from \hyperref[sec:regret-prelim]{Section \ref*{sec:regret-prelim}}}

\subsection{Algorithm and its Running Time} \label{app:running_regret}

While our main focus is to provide sublinear regret bounds for smoothed online learning our analysis also provides an algorithmic solution describe below.

\begin{algorithm}[H]
    \SetAlgoLined
    \caption{Smooth Online Learning}
    \label{alg:QueryAnswering}
    \DontPrintSemicolon
    %\LinesNumbered
    \KwIn{Instance Space $\X$, Hypothesis Class $\mathcal{H}$, Smoothness parmeter $\sigma$, Time horizon $T$}
    \;  
    \emph{Cover Construction:} Compute $ \mathcal{H}'  \subseteq  \mathcal{H} $ that is a $ \gamma $-cover of $\H$ with respect to the uniform distribution on $\X$ for $\gamma = \frac{\sigma  }{4\sqrt{T}}$. \;
    % \SetNoLine 
    \For{$t  = 1 \dots T $}{
Use a standard online learning algorithm, such as Hedge, on $\H'$ to pick an $h_t$, where the history of the play is $\{s_\tau\}_{\tau<t}$ and $\{h_\tau\}_{\tau < t}$ \; 
        Receive $s_t = \left( x_t , y_t \right)$ and suffer loss $\err_{s_t} \left( h_t \right) $. 
    }
\end{algorithm}

The running time of the algorithm comprises of the initial construction  {of $\H'$} and then running a standard online learning algorithm on $\H'$.

Standard online learning algorithms such as Hedge and FTPL take time polynomial in the size of the cover since in standard implementations they maintain a state corresponding to each hypothesis in $\H'$. 
In our setting, the size of the  cover is $(  41 \sqrt{T} / \sigma)^d $. 

The time required to construct a cover depends on the access we have to the class. 
One method is to randomly sample a set $S$ with $m= O( \vc\left( \H \right) T/ \sigma^2 )$ points from the domain uniformly and construct all possible labelings on this set induced by the class. 
The number of labellings of $S$ is bounded by $O(m^{ \vc\left( \H \right)} )$ by the Sauer--Shelah lemma.  The cover is constructed by then finding functions in the class $\H$ that are consistent with each of these labellings.  This requires us to be able to find an element in the class consistent with a given labeling, which can be done by a ``consistency'' oracle.  Naively, the above makes $2^m$ calls to the consistency oracle, one for each possible labeling of $S$.

The above analysis and runtime can be improved in several ways. First, $\H'$ can be constructed in time $O(m^{ \vc\left( \H \right)} )$  rather than $2^m$. This can be done by constructing the cover in a hierarchical fashion, where the root includes the unlabeled set $S$ and at every level one additional instance in $S$ is labeled by $+1$ or $-1$. At each node, the consistency oracle will return a function $h\in \H$ that is consistent with the labels so far or state that none exists. Nodes for which no consistent hypothesis so far exists are pruned and will not expand in the next level. Since the total number of leaves is the number of ways in which $S$ can be labeled by $\H$, i.e., $O(m^d)$, the number of calls to the consistency oracle is $O(m^d)$ as well.  
The runtime of standard online learning algorithms can also be improved significantly when an empirical risk minimization oracle is available to the learner, in which case a runtime of $O(\sqrt{ |\H'|})$  for general classes~\citep{HazanKoren} or even $\mathrm{polylog}(|\H'|) )$ for structured classes ~\citep{Oracle_Efficient}  is possible.

\subsection{Proof of \hyperref[lem:union-ada]{Lemma \ref*{lem:union-ada}}} \label{app:Proofoffinite} 

At a high level, note that any $f\in \F$ has measure at most $\epsilon/\sigma$ on any (even adaptively chosen) $\sigma$-smooth distribution. Therefore, for any fixed $f$, $\E_\adist[\sum_{t=1}^T f(x_t)] \leq T\epsilon/\sigma$. To achieve this bound over all $f\in \F$, we take a union bound over all such functions.

More formally, for any $s$
\begin{align*}
\exp\left(s\E_\adist\left[\max_{f\in \F} \sum_{t=1}^T f(x_t) \right] \right) 
& \leq \E_\adist\left[ \exp\left( s \max_{f\in \F} \sum_{t=1}^T f(x_t) \right)  \right] \qquad \text{(Jensen's inequqlity)} \\
& \leq \E_\adist\left[\max_{f\in \F} \exp\left(s  \sum_{t=}^T f(x_t) \right)  \right] \qquad \text{(Monotonicity of $\exp$)} \\
& \leq \sum_{f\in \F}\E_\adist\left[\exp\left( s \sum_{t=1}^T f(x_t) \right)  \right] . \numberthis \label{eq:exp-union}
\end{align*}
Consider a fixed $f\in \F$. Note that even when the choice of a $\sigma$-smoothed distribution $\D$ depends on earlier realizations of $x_1, \dots, x_{i-1}$, $\Pr_{x_i \sim \D}[f(x_i)] \leq \frac{\epsilon}{\sigma}$. Therefore, $\sum_{t=1}^T f(x_t)$ for $\vec x\sim \adist$ is stochastically dominated
by that of a binomial distribution $Bin(T, \epsilon/\sigma)$. 
Note that $\exp(\cdot)$ is a monotonically increasing functions and let $p = \epsilon/\sigma$. We have
\begin{equation}
\label{eq:fixed-f-exp}
\E_\adist\left[\exp\left(s \sum_{t=1}^T f(x_t) \right)  \right] \leq \sum_{v=0}^T \exp(sv) {T \choose v} p^v (1-p)^{T-v} = \big(p (\exp(s) - 1) +1 \big)^T.
\end{equation}
Combining Equations~\eqref{eq:exp-union} and \eqref{eq:fixed-f-exp} and noting that $\ln(1+x) \leq x$, we have
\begin{align*}
\E_\adist\left[\max_{f\in \F} \sum_{t=1}^T f(x_t) \right] &\leq \frac{\ln(|\F|) + T p\left(\exp(s) - 1 \right)}{s}.
\end{align*}
Let $s = \sqrt{\ln(|\F|)}/Tp$. Note that because $s\in (0,1)$, we have $\exp(s) \leq 1 + 2s$. Hence, by replacing $s$ in the above  inequality we have 
\begin{align*}
\E_\adist\left[\max_{f\in \F} \sum_{t=1}^T f(x_t) \right] \in O\left(Tp \sqrt{\ln(|\F|)} \right).
\end{align*}

\subsection{Proof of \autoref{thm:main-ada} }

Consider any hypothesis class $\H'$ and an algorithm that is no-regret with respect to any adaptive adversary on hypotheses in $\H'$. It is not hard to see that 
\begin{align*}
\E[\regret(\A, \adist)] &= \E_{\vec s\sim \adist} \left[\sum_{t=1}^T \err_{s_t}(h_t) - \min_{h\in \H}  \err_{s_t}(h_t)  \right]  \\
& \leq \E_{\vec s\sim \adist} \left[\sum_{t=1}^T \err_{s_t}(h_t) - \min_{h\in \H'}  \sum_{t=1}^T\err_{s_t}(h)  \right] + \E_{\vec s\sim \adist}\left[\min_{h'\in \H'} \sum_{t=1}^T \err_{s_t}(h') - \min \sum_{t=1}^T \err_{s_t}(h) \right]\\
& \leq O\left( \sqrt{T \ln(|\H'|)} \right) + \E_{\adist}\left[\max_{h\in \H}\min_{h'\in \H'} \sum_{t=1}^T 1\left(h(x_t) \neq h'(x_t) \right)  \right]. \numberthis \label{eq:regret-total}
\end{align*}
Therefore, it is sufficient to choose an $\H'$ of moderate size such that every function $h\in \H$ has a proxy $h'\in \H'$ even when these functions are evaluated on instances drawn from a \emph{non-iid and adaptive sequence of smooth distributions}. We next describe the choice of $\H'$. 

Let $\H'$ be a $\frac \epsilon 2$-net of $\H$ with respect to the uniform distribution $\U$, for an $\epsilon$ that we will determine later. Note that 
 any $\epsilon$-bracket with respect to $\U$ is also an $\epsilon$-net, so $|\H'| \leq \bnumber\H\U {\epsilon/2}$.\footnote{Alternatively, we can bound $|\H'|\leq (41/\epsilon)^{\vc(\H)}$ by \cite{ HAUSSLER1995217}.}
Let $\G$ be the set of symmetric differences between $h\in \H$ and its closest proxy $h'\in \H'$, that is,
\[ \G = \{g_{h, h'}(x) = 1(h(x)\neq h'(x)) \mid \forall h\in \H \text{ and } h'\in \H', \text{ s.t. } \E_\U[g_{h,h'}(x)]\leq \epsilon/2 \}.
\]
Note that because $\G$ is a subset of all the symmetric differences of two functions in $\H$, by Theorem~\ref{thm:k-fold} its bracketing number is bounded as follows.
\begin{equation}
\label{eq:bnumber-delta}
 \bnumber{\G}{\U}{\epsilon/2} \leq \left(\bnumber{\H}{\U}{\epsilon/4}\right)^4. 
\end{equation}

Let $\B(\G)$ be the set of upper $\epsilon/2$-brackets of $\G$ with respect to $\U$, i.e., for all $g\in G$, there is $b\in \B(\G)$ such that for all $x\in \X$, $g(x) \leq b(x)$ and $\E_\U[b(x) - g(x)]\leq \epsilon/2$. Note that 
\[
\E_{\adist}\left[ \max_{h\in \H}\min_{h'\in \H'} \sum_{t=1}^T 1\left(h(x_t) \neq h'(x_t) \right)  \right] = \E_{\adist}\left[\max_{g\in \G} \sum_{t=1}^T g(x_t) \right] \leq \E_{\adist}\left[\max_{b\in \B(\G)} \sum_{t=1}^T b(x_t)  \right],
\]
where the last transition is by the fact that $\B(\G)$ includes all upper brackets of $\G$.

We now note that $\B(\G)$ meets the conditions Lemma~\ref{lem:union-ada}, namely because all $g\in \G$ have measure at most $\epsilon/2$ over $\U$ and $\B(\G)$ is the set of $\epsilon/2$-upper brackets of $\G$, we have that $\E_\U[b(x)]\leq \epsilon$ for all $b\in \B(\G)$. Therefore, by Lemma~\ref{lem:union-ada} and Equation~\ref{eq:bnumber-delta}, we have 
\[\E_{\adist}\left[\max_{b\in \B(\G)} \sum_{t=1}^T b(x_t)  \right] \leq O\left(  T\frac \epsilon \sigma  \sqrt{\ln\left( \bnumber{\H}{\U}{\epsilon/4} \right)}  \right)
\]
Replacing this in Equation~\ref{eq:regret-total} we have that 
\[ 
\E[\regret(\A, \adist)]\in O\left(  \sqrt{T \ln\left( \bnumber \H\U{\epsilon/4}\right) } +  T\frac \epsilon \sigma\sqrt{\ln\left( \bnumber{\H}{\U}{\epsilon/4} \right)} \right)
\]
Choosing $\epsilon = \sigma/\sqrt{T}$ proves the claim.

\subsection{Proof of \autoref{thm:embeddable}}
        Consider the map $\psi : \mathcal{X} \to \br^m $ that embeds $\mathcal{G}$ in $m$ dimensions and let $\H$ be the class of halfspaces in $\br^m$.
        We want to bound the bracketing number of $\mathcal{G}$ by that of $\H$.
       Let $\B(\H) =\{ [ h_i , h^i ]\}_{i}$ be an $\epsilon$-bracketing for $\H$ with respect to a measure $\mu$ that we will specify later. 
Consider the set of brackets $\B' = \{ [ h_i \circ \psi  , h^i \circ \psi  ]  \mid \text{ for all } [ h_i  , h^i ]\in \B(\H) \} $.
We first argue that $\B'$ is a bracketing for $\G$ with respect to $\nu$.
       To see this, note that any $ g \in \mathcal{G} $ can be expressed as $g = h \circ \psi$ for some halfspace $h$.
        Considering the bracket $[ h_i , h^i  ] \ni h$ in $\B(\H)$. Note that $h_i\circ \psi \preceq  h \circ \psi \preceq h^i \circ \psi$ and thus $g\in [ h_i \circ \psi , h^i \circ \psi ]$. We next argue that these are $\epsilon$-brackets under measure $\nu$. Let $\mu$ be the measure such that to sample $z\sim \mu$ we  first sample $x \sim \nu$ and let $z = \psi\left( x \right)$. Note that
        \begin{equation*}
            \Pr_{x \sim \nu} \left[ h^i \left( \psi \left( x \right) \right) \neq  h_i \left( \psi\left( x \right) \right) \right] = \Pr_{z\sim \mu} \left[ h^i \left( z \right) \neq  h_i \left( z \right) \right] \leq \epsilon,
        \end{equation*}
  where the last transition is by the fact that $\B(\H)$ is an $\epsilon$-bracketing for $\H$ with respect to $\mu$. This concludes that $\B'$ is an $\epsilon$-bracketing for $\G$ with respect to $\nu$. We complete the proof by using  \autoref{thm:half_brack} to bound $|\B'| = |\B(\H)| \leq (\nicefrac{m}{\epsilon})^{O\left( m \right)} $.

\subsection{Proof of \autoref{thm:k-fold}}

We first consider the case of $k=2$ and then extend our argument to general $k$. Let $\epsilon' = \epsilon/k$ and 
let $\B(\F_1)$ and $\B(\F_2)$ be $\epsilon'$-bracketings for $\F_1$ and $\F_2$, respectively.

For $\F_1 \cdot \F_2$, construct $\B =  \left\{ [f_\ell \cap g_\ell, f^u \cap g^u] \mid \text{for all } [f_\ell, f^u]\in \B(\F_1) \text{ and } [g_\ell, g^u]\in \B(\F_2)\right\}$. 
First note for any $f_1\in \F_1$ and $f_2\in \F_2$, $f_1\cap f_2$ is included in one of these brackets. In particular, for brackets $[f_\ell, f^u]\ni f_1$ and $[g_\ell,g^u]\ni f_2$, we have that $f_\ell \cap g_\ell \preceq f_1\cap f_2 \preceq f^u \cap g^u$ and $[f_\ell \cap g_\ell, f^u \cap g^u] \in \B$. Furthermore, 
\begin{align*}
            \Pr_{x \sim \mu} \left[  \left( f_\ell(x) \cap g_\ell( x) \right) \neq  \left( f^u(x) \cap g^u(x) \right) \right] &\leq \Pr_{x \sim \mu} \left[ \left(  f_\ell(x) \cap g_\ell(x)  \right) \neq \left(  f_\ell(x) \cap g^u ( x) \right) \right] \\
            & \hphantom{=} + \Pr_{x \sim \mu} \left[  \left( f_\ell(x) \cap g^u ( x) \right) \neq  \left( f^u(x)\cap g^u ( x)  \right) \right] \\ 
            & \leq 2 \epsilon'.
\end{align*}
Therefore, $\B$ is a $2\epsilon'$-bracketing for $\F_1\cdot \F_2$ of size $ \bnumber{\F_1}{\mu}{\epsilon'} \cdot  \bnumber{\F_2}{\mu}{\epsilon'}$.         Repeating this inductively and using $\epsilon' = \epsilon/k$, we get the claim for $k$ classes.

Similarly, for $\F_1+\F_2$, construct $\B =  \left\{ [f_\ell \cup g_\ell, f^u \cup g^u] \mid \text{for all } [f_\ell, f^u]\in \B(\F_1) \text{ and } [g_\ell, g^u]\in \B(\F_2)\right\}$. 
First note for any $f_1\in \F$ and $f_2\in \F_1$ and 
their respective brackets $[f_\ell, f^u]\ni f_1$ and $[g_\ell,g^u]\ni f_2$, we have that $f_\ell \cup g_\ell \preceq f_1\cup f_2 \preceq f^u \cup g^u$ and $[f_\ell \cup g_\ell, f^u \cup g^u] \in \B$. Furthermore, 
      \begin{align*}
            \Pr_{x \sim \mu} \left[ \left( f_\ell(x) \cup g_\ell( x) \right) \neq  \left( f^u(x) \cup g^u(x) \right) \right] &\leq   \Pr_{x \sim \mu} \left[  f_\ell  \left( x \right) \neq  f^u  \left( x \right) \right] +  \Pr_{x \sim \mu} \left[  g_\ell \left( x \right) \neq g^u \left( x \right) \right] \\
                & \leq 2 \epsilon'. 
        \end{align*}
        Therefore, $\B$ is a $2\epsilon'$-bracketing for $\F_1+\F_2$ of size $ \bnumber{\F_1}{\mu}{\epsilon'} \cdot  \bnumber{\F_2}{\mu}{\epsilon'}$.        
Repeating this inductively and using $\epsilon' = \epsilon/k$, we get the claim for $k$ classes.

As for the $\G $, the set of all symmetric differences, note that $ f_1 \Delta f_2 = \left( f_1 \cup f_2 \right) \setminus \left( f_1 \cap f_2 \right) = \left( f_1 \cup f_2 \right) \cap \overline{\left( f_1 \cap f_2 \right)} $. Furthermore, for any class $\F$, the class $\overline{\F}=\{\overline{f} \mid \forall f\in \F \}$ has the same bracketing number as $\F$. Therefore, the bracketing number of $\G$ follows from using the bracketing number $\F+\F$, $\overline{\F+\F}$, and their intersection.

\subsection{Proof of \hyperref[cor:bnumber]{Corollary \ref*{cor:bnumber}}}
The set of polynomial threshold functions in $n$ variables and of degree $d$ is embeddable as halfspaces in $O(n^d)$ dimensions using the map 
\[
\phi \left( x_1 , \dots , x_n \right) = \left(  \prod_{i\in S} x_i  \right)_{S\in \{1,\dots, n\}^{\leq d}},
\]
    which maps variables to all monomial of degree $d$. It can be seen that the number of monomials of degree at most $d$ in $n$ variables is given by $ \binom{n+d+1}{d+1} $ which is approximately $O\left(n^d \right)$ when $d$ is small.
Combining \autoref{thm:embeddable} and \autoref{thm:half_brack} completes the proof for polynomial threshold functions.

A $k$-polytope in $\br^n$ is an intersection of $k$-halfspaces in $\br^n$. Combining \autoref{thm:k-fold} and \autoref{thm:half_brack} completes the proof.

\section{More Details on Bracketing Number and Sign Rank} \label{app:brackextra}

Though bracketing numbers are a fundamental concept in statistics, until recently their connection to VC theory was not well understood. 
\cite{adams2010,adams2012} show that for 
countable (can be generalized to classes that are well approximated by countable classes) classes with finite VC dimension the bracketing numbers with respect to  any measure is finite ({this establishes what is known as a universal Gilvenko--Cantelli theorem under ergodic sampling.)}  
\begin{theorem}[Finite Bracketing Bounds for VC Classes]
    Let $\mathcal{C}$ be a countable class with finite VC dimension. 
    Then, $ \bnumber{ \mathcal{C}}{\mu}{\epsilon}  < \infty$. 
\end{theorem}
Though the above theorem proves that $\epsilon$-bracketing numbers are finite, their growth rate in $1/\epsilon$  can be arbitrarily large.
See \cite{UGC} for some interesting examples of classes where the bracketing numbers grow arbitrarily fast.  

Another combinatorial quantity that can help bound the regret in presence of adaptive smooth adversaries is \emph{sign rank}.

\begin{definition}[Sign Rank]
    Let $\mathcal{X}$ be an  instance space and let $\mathcal{F} $ be a class. 
    We can denote the class naturally as $\{-1,1\}$-valued $ \mathcal{X} \times \mathcal{F} $ matrix $M_{\mathcal{F}}$ where the entry corresponding to $\left( x , f \right)$ is $ f\left( x \right) $. 
    The sign rank of a class is the highest rank of a real matrix that agrees with a finite submatrix of $M_{\mathcal{F}}$ in sign. 
    If this is unbounded, the class is said to have infinite sign rank.  
\end{definition}

    The sign rank of a class captures the dimension in which the class can be embedded as thresholds. 
    
    \begin{fact}[Sign Rank Embedding, see e.g. \cite{Satya_LinearComplexity}]
        The sign rank of a class corresponds to the smallest dimension $d$ that the class can be embedded as thresholds. 
    \end{fact}
    \autoref{thm:embeddable} effectively says that classes with small sign rank have a slowly growing bracketing numbers and thus have low regret in the smoothed online learning setting.
    Thus, the complexity of smoothed online learning lies somewhere in between the sign rank and VC dimension.   
 On the other hand, it is known that even classes with small VC dimension can have arbitrarily large sign rank \citep{AlonWezlHaussler,Embedding,Signrank}. 
        An intermediate  question is whether classes with slow growing bracketing number also have good sign rank. 
    It would be interesting to characterize the complexity of smoothed online learning in terms of either the sign rank or bracketing numbers.

\section{Query Answering} \label{sec:DPProofs}
 
\subsection{Smooth MWEM Algorithm}
\begin{algorithm}[H]
    \SetAlgoLined
    \caption{Smooth Multiplicative Weights Exponential Mechanism}
    \label{alg:QueryAnswering}
    \DontPrintSemicolon
    \KwIn{ Universe $\X$ with $\abs{\X} = N $, Data set $B$ with $n$ records, Query set $\mathcal{Q}$, Privacy parameters $\epsilon$ and $\delta$, Smoothness parameter $\sigma$.}
    \;
       Let ${\D_0} \left( x \right) = \nicefrac{1}{N} $ for all $x\in \X$.\;  
    \emph{Cover Construction:} Compute $ \mathcal{Q}'  \subseteq  \mathcal{Q} $ that is a $\gamma$-cover of $\Q$ with respect to the uniform distribution for $\gamma = \frac{\sigma}{2n} $. \;
    \For{$i  = 1 \dots T $}{
     	\emph{Exponential Mechanism:} Sample $q_i \in \mathcal{Q}' $ according to the exponential mechanism with parameter $\nicefrac{\epsilon}{2T}$ and score function 
        \begin{equation*}
            s_{i}( \D_B, q)=n \left|q\left(\D_{i-1}\right)-q( \D_B )\right|. 
        \end{equation*} \; 
        \emph{Laplace Mechanism:} Let $m_i  = q_i  \left( \D_B \right) + \frac{1}{n} Lap \left( \nicefrac{2T}{\epsilon} \right) $ .\; 
        \emph{Multiplicative Update:} Update $\D_{i-1}$ using the rule 
        \begin{equation*}
            \D_i\left( x \right) \propto \D _{i-1}\left( x \right)   \exp\left(  \frac{ q_{i} \left( x \right)  \left(   m_i - q_i( \D_{i-1} ) \right)  }{2 }  \right) . 
        \end{equation*}
    }
    Let $ \overline{\D} = \frac{1}{T}  \sum_{i=1}^T \D_{i}   $.  \; 
    
    \KwOut{For each $q \in \mathcal{Q} $, answer with $v_q = q'\left( \overline{\D} \right)$ where $q'$ is the closest function in $ \mathcal{Q}'$ to $q$.}  
\end{algorithm}

\subsection{Proof of Theorem~\ref{thm:DPMW}} \label{app:datareleaseanalysis}
In this section we prove the following theorem. 

\medskip
\textbf{Theorem~\ref{thm:DPMW} (restated).~}\emph{
    For any $(\sigma,0)$-smooth dataset $B$ of size $n$, a query class $\Q$ with VC dimension $d$, $T \in \mathbb{N}$ and $\epsilon>0$, \hyperref[alg:QueryAnswering]{Smooth Multiplicative Weights Exponential Mechanism} is $\epsilon$-differentially private and with probability at least $1 - 2T  \left(  \nicefrac{\gamma}{41} \right)^{ \vc\left( \mathcal{Q} \right) }  $, calculates values  $v_{q}$  for all $q\in \Q$ such that 
    \begin{equation*}
        \max_{q \in \Q }\left\{   \abs{ v_q - q\left(\D_B \right) } \right\} \leq \frac{1}{n} + 2  \sqrt{\frac{\log\left( \nicefrac{1}{\sigma} \right) }{T}}+\frac{10 T d \log\left( \nicefrac{2n}{\sigma} \right) }{\epsilon n }.
    \end{equation*}
}

Let us first provide a few useful lemmas.

\begin{lemma}[Cover under Smoothness] \label{lem:SmoothCover}
Let $B$ be $(\sigma, 0)$-smooth data set.
    Let $ \mathcal{Q} '\subseteq \Q$ be a $\gamma$-cover of $\mathcal{Q}$ under the uniform distribution. 
    For a $q \in \mathcal{Q} $, let $ q' \in \mathcal{Q} $ be such that $ \Pr_{x \sim  \U  } \left[ q\left( x \right) \neq q' \left( x \right) \right] \leq \gamma $. Then, 
    \begin{equation*}
        \abs{ q \left( \D_B \right) - q' \left( \D_B \right) } \leq  \frac{2 \gamma}{\sigma }. 
    \end{equation*}
\end{lemma}
\begin{proof} \label{pf:SmoothCover}
    From the $ \left(  \sigma, 0 \right)$-smoothness of $B$, we get  
    \begin{align*}
         \abs{ q \left( \D_B \right) - q' \left( \D_B \right) } & = \abs{ q \left( \overline {\D_B} \right) - q' \left( \overline{\D_B} \right) } \\ 
         & \leq  \sum_{x \in D} \abs{   \left(  q\left( x \right) - q' \left( x \right) \right) }\,  \overline{\D_B}(x)    \\ 
         & \leq    \sum_{x \in \X } 2 \mathbb{I} \left( q(x) \neq q'\left( x \right) \right) \overline{\D_B}(x)\\
         & \leq  \frac{2}{ \sigma} \sum_{x \in \X }  \mathbb{I} \left( q(x) \neq q'\left( x \right) \right) \U\left( x \right) \\ 
         & \leq \frac{2}{\sigma} \Pr_{x \sim \U} \left[ q\left( x \right) \neq q' \left( x \right) \right] \\ 
         & \leq \frac{2 \gamma}{\sigma } 
    \end{align*}
    as required.
\end{proof}

Define the potential function $\Psi_{i}=\sum_{x \in  \X } \overline{ \D_{B} } (x) \log \left( \overline{ \D_{B} } (x)  / \D_{i}(x)\right) $, where $\overline{\D_B}$ is a corresponding $\sigma$-smooth distribution that matches the query answers for the $(\sigma, 0)$-smooth data set $B$.
Here we make a few observations about the potential function. 

\begin{fact} \label{fact:potentialfacts}
    For all $i \leq T$, we have $\Psi_i \geq 0.$    Furthermore, 
  $\Psi_0 \leq \log  \frac{1}{\sigma}.$
    As a result, $\Psi_0 - \Psi_T \leq \log  \frac{1}{\sigma}.$
\end{fact}
    \begin{proof}
        The first claim follows from the positivity of the KL divergence. For the second one, recall that from the $\sigma$-smoothness of $\D_B$ and the fact that $\D_1$ is the uniform distribution, we have $ \D_B \left( x \right) \leq \sigma^{-1} \D_0 \left( x \right)  $ for all $x \in \mathcal{X} $.   
        \begin{align*}
            \Psi_0 &= \sum_{x\in \X } \overline{\D_B} \left( x \right) \log \frac{ \overline{\D_B} \left( x \right) }{\D_0 \left( x \right) }  \leq  \sum_{x\in \X } \overline{\D_B} \left( x \right) \log \frac{1 }{ \sigma }  = \log \frac{1}{\sigma}
        \end{align*}
        as required. 
    \end{proof}
Below is a direct adaptation of a result of  \citet{NIPS2012_4548} for bounding the change in the potential functions.
    \begin{lemma}[Lemma A.4 in \cite{NIPS2012_4548}] \label{lem:potential_improve}
        \begin{equation*}
            \Psi_{i-1} -{\Psi}_{i} \geq\left(\frac{q_{i}\left( \D_{i-1}\right)-q_{i}( \overline{ \D_{B} } )}{2 }\right)^{2}-\left(\frac{m_{i} - q_{i}( \overline{ \D_{B} }) }{2  }  \right)^{2}. 
        \end{equation*}
    \end{lemma}

    \begin{lemma}[Exponential and Laplace Mechanism guarantees]  \label{lem:ExpLapcon}
        With probability at least $1 - \nicefrac{2T}{ \abs{\mathcal{Q} ' } } $, we have 
     
        \begin{equation*}
            \abs{q_i\left( \D_{i - 1} \right) - q_i \left( \D_B \right) } \geq \max_{q' \in \mathcal{Q}'} \left\{ q'\left( \D_i \right) - q' \left( \D_B \right)  \right\} - \frac{8T \log \abs{\mathcal{Q}'} }{\epsilon n}
        \end{equation*}
        and 
        \begin{equation*}
            \abs{m_i - q_i \left( \D_B \right) } \leq \frac{2T \log \abs{\mathcal{Q}'} }{\epsilon n }.
        \end{equation*}
    \end{lemma}

Here we recall again the error guarantees from \cite{NIPS2012_4548}. 

\begin{theorem}[\cite{NIPS2012_4548}] \label{thm:MWEM_App}
    For any data set $B$ of size $n$, a finite query class $\Q$, $T \in \mathbb{N}$ and $\epsilon>0$, MWEM is $\epsilon$-differentially private and with probability at least $1 - \nicefrac{2T}{|\mathcal{\Q}|}$ produces a distribution  $\overline{\D}$ over $\X$ such that 
    
    \begin{equation*}
        \max_{q \in \Q }\left\{   \abs{q\left(\overline{\D}\right) - q\left(\D_B \right) } \right\} \leq 2  \sqrt{\frac{\log |\mathcal{X}|}{T}}+\frac{10 T \log |\mathcal{Q}|}{\epsilon n} .
    \end{equation*}
\end{theorem}

\begin{proof}[Proof of \autoref{thm:DPMW}]
    % \begin{proof}
               Our proof closely resembles that of  \autoref{thm:MWEM_App} from \cite{NIPS2012_4548}.
        Note that since $B$ is $\left( \sigma , 0 \right)$-smooth, we have a $\sigma$-smooth distribution $ \overline{\D_B }   $ with $ \overline{\D_B } \left( x \right) \leq \frac{1}{\sigma N}$ such that for all $q \in \mathcal{Q}$, $ q\left( \mathcal{D}_B \right) = q\left(  \overline{ \mathcal{D}_B} \right)  $.
       Furthermore, note that we chose a cover $\Q'\subseteq \Q$. Therefore, $ q'\left( \mathcal{D}_B \right) = q'\left(  \overline{ \mathcal{D}_B} \right)$ holds for all $q' \in \mathcal{Q}' $ as well.

Note that since $ q'\left( \D_B  \right) = q'\left( \overline{\D_B} \right)  $ for all $q' \in \mathcal{Q}'  $, we can replace this in the above equation. 
    For the sake of completeness, we sketch the rest of the proof. 
    From Jensen's inequality, we have 
    \begin{equation}
\label{eq:app:jensen}
        \max_{q' \in \mathcal{Q}' } \abs{ q'\left(\overline{\D} \right) - q'\left(  \D_{B}  \right) } \leq \frac{1}{T } \sum_{i = 1}^T \max_{ {q'} \in \mathcal{Q}' } \abs{ q'\left( \D_i \right) - q'\left(  \D_{B}  \right) }.
    \end{equation}
 From \hyperref[lem:ExpLapcon]{Lemma \ref*{lem:ExpLapcon}} and \hyperref[lem:potential_improve]{Lemma \ref*{lem:potential_improve}}, we get that with probability at least $1- 2T / \abs{\mathcal{Q'}} $, we get   
    \begin{equation*}
        \Psi_{i-1} -{\Psi}_{i} \geq\left(\frac{ \max_{q' \in \mathcal{Q}'} \left\{ q'\left( \D_i \right) - q' \left( \D_B \right)  \right\} - \frac{8T \log \abs{\mathcal{Q}'} }{\epsilon n}}{2 }\right)^{2}-\left(\frac{ T\log \abs{ \mathcal{Q} } }{ \epsilon n }  \right)^{2}. 
    \end{equation*}
    Rearranging this and taking the average, we get
   \begin{equation*}
        \frac{1}{T }\sum_{i = 1}^T \max_{q' \in \mathcal{Q}' } \abs{ q'\left( \D_i \right) - q'\left( \D_B \right) } \leq  \frac{1}{T }\sum_{i = 1}^T\left[  \sqrt{ 4 \left( \Psi_{i - 1} - \Psi_i \right) + \frac{4T^2 \log^2 \abs{\mathcal{Q}' } }{ n^2\epsilon^2}  } + \frac{8T \log \abs{\mathcal{Q}' } }{ n\epsilon}  \right] .
    \end{equation*}
    
    Applying the concavity of the square root function i.e., $\frac 1T \sum_{i=1}^T (x_i)^{1/2} \leq \left( \frac 1T \sum_{i=1}^{T } x_i \right)^{1/2}$, 
      \begin{align*}
        \frac{1}{T }\sum_{i = 1}^T \max_{q \in \mathcal{Q}' } \abs{ q'\left( \D_i \right) - q'\left( \D_B \right) } &\leq \sqrt{ \sum_{i = 1}^T \frac{4 \left( \Psi_{ i - 1} - \Psi_i \right) }{T }+ \frac{4T^2 \log^2 \abs{\mathcal{Q}' } }{ n^2\epsilon^2}  } + \frac{8T \log \abs{\mathcal{Q}' } }{ n\epsilon} \\
        &\leq \sqrt{  \frac{4 \left( \Psi_0 - \Psi_T \right) }{T }+ \frac{4T^2 \log^2 \abs{\mathcal{Q}' } }{ n^2\epsilon^2}  } + \frac{8T \log \abs{\mathcal{Q}' } }{ n\epsilon} \\
            & \leq  \sqrt{ \frac{4\log\left( \frac{1}{\sigma} \right) }{ {T}}+ \frac{4T^2 \log^2 \abs{\mathcal{Q}'} }{ n^2\epsilon^2}  } + \frac{8T \log \abs{\mathcal{Q}' } }{ n\epsilon} \\ 
            & \leq  2  \sqrt{\frac{ \log\left( \frac{1}{\sigma} \right) }{T}} + \frac{10T \log \abs{\mathcal{Q}' } }{ n\epsilon}. 
        \end{align*}
The second inequality follows by summing the telescoping series. 
The third follows from \hyperref[fact:potentialfacts]{Fact \ref*{fact:potentialfacts}}. 
The last equation follows from the fact that $ \sqrt{x+y} \leq \sqrt{x} + \sqrt{y} $ for all positive $x,y$.  
{Using Equation~\ref{eq:app:jensen} and the fact that
$ \abs{\mathcal{Q}}' \leq \left( \nicefrac{41}{\gamma} \right)^d$ we have}
\begin{equation*}
    \max_{q' \in \mathcal{Q}' } \abs{ q'\left( \overline{\D} \right) - q'\left( \D_B \right) } \leq 2  \sqrt{\frac{ \log\left( \nicefrac{1}{\sigma} \right) }{T}} + \frac{10 T d \log\left( \nicefrac{2n}{\sigma} \right) }{\epsilon n }.
\end{equation*} 
 {Let
$v_q = q'(\overline{\D})$ for 
$q'\in \Q'$ that is the closest hypothesis to $q$ with respect to the uniform distribution.} Then
\begin{align*}
    \abs{  q\left( \D_B  \right) - v_q   } & = \abs{  q\left( \D_B \right)  - q'\left( \D_B \right)    + q' \left( \D_B \right)  - q' \left(\, {\overline{\D}} \right)  }  \\
    & \leq \abs{  q\left( \D_B \right)  - q'\left( \D_B \right) }   +\abs{  q' \left( \D_B \right) -q'\left( {\overline{\D}} \right) } \\
    & \leq \frac{2 \gamma }{\sigma} + 2  \sqrt{\frac{\log \nicefrac{1}{\sigma}  }{T}}+\frac{10 T  d \log \left( \nicefrac{41}{ \gamma} \right)     }{\epsilon n } .
\end{align*}
Setting $\gamma =  \frac{\sigma}{ 4 n }  $, we get the desired result. 
\end{proof}

Setting $T = \epsilon^{2/3} n^{2/3} \log ^{1/3} \left( 1 / \sigma \right)  d^{-2/3} \log^{-2/3} (2 n / \sigma) $, we get $\left( \epsilon , 0 \right)$ differential privacy with 
\begin{equation*}
    \max_{q\in \mathcal{Q}} \left\{   \abs{ {v_q} - q\left( \D_B \right) }    \right\} \leq O\left( \sqrt[3]{ \frac{d \log\left( \nicefrac{1}{\sigma} \right) \log \left(\nicefrac{2n}{\sigma} \right)  }{n \epsilon} } \right). 
\end{equation*}
Also, as noted in \cite{NIPS2012_4548}, one can use adaptive $k$-fold composition (see e.g. \cite{AlgorithmicFoundationsDp}) to get $\left( \epsilon , \delta \right)$-differential privacy with

\begin{equation*}
    \max_{q\in \mathcal{Q}} \left\{   \abs{ {v_q} - q\left( \D_B \right) }    \right\} \leq O\left( \sqrt{ \frac{d}{\epsilon n}    \log ^{ \frac{1}{2} }  \left(\frac{1}{\sigma} \right)  \log\left( \frac{n}{\sigma } \right) \log\left( \frac{1}{ \delta}   \right)} \right). 
 \end{equation*}

\allowdisplaybreaks
\subsection{Running Time of the Algorithm} \label{sec:running_SMWEM}
    The running time of the algorithm is similar to the running time of the MWEM algorithm of \cite{NIPS2012_4548}. 
    The main additional step is the construction of the cover $ \mathcal{Q}'$.
    Similar to \hyperref[app:running_regret]{ Appendix \ref*{app:running_regret} }, this cover can be constructed in time $O\left(|\Q'|\right)$.
    The exponential mechanism requires $ O\left( n \abs{ \mathcal{Q} }' \right) $ to evaluate all the queries on the cover and time $O\left(  \abs{ \mathcal{Q} }' \abs{\X} \right)$ to execute each iteration of the algorithm. 
    Recall that $|\Q'| \leq \left(41 n / \sigma \right)^d$, thus the running time is bounded by $O \left(n \left( 41 n / \sigma \right)^d + T   \left( 41 n / \sigma \right)^d \abs{\X} \right).$

This runtime can also be improved using several theoretical tricks, e.g., $q(\D_i)$ can be approximated by taking random points from $\D_i$ in time that is independent of $\X$. 
    
    Note that the runtime of our algorithm improves upon the runtime of MWEM by using smaller query sets. As noted in \cite{NIPS2012_4548}, their algorithm is amenable to many optimizations and modifications that make it very fast and practical~\cite{NIPS2012_4548}.
\section{Data Release} \label{sec:DPDataRelease}

\subsection{Projected Smooth MWEM Algorithm}
    \begin{algorithm}[H]
         \SetAlgoLined
        \caption{Projected Smooth Multiplicative Weight Exponential Mechanism}
        \label{alg:DataRelease}
        \DontPrintSemicolon

        \KwIn{Universe $\X$ with $\abs{\X} = N $, Data set $B$ with $n$ records, Query set $\mathcal{Q}$, Privacy parameters $\epsilon$ and $\delta$, Smoothness parameter $\sigma$.}
        \;
       Let $\D_0 \left( x \right) = \nicefrac{1}{N} $ for all $x\in \X$.\; 
    \emph{Cover Construction:} Compute $ \mathcal{Q}'  \subseteq  \mathcal{Q} $ that is a $\gamma$-cover of $\Q$ with respect to the uniform distribution for $\gamma = \frac{\sigma}{2n} $. \;
        \For{$i  = 1 \dots T $}{
            \emph{Exponential Mechanism:} Sample $q_i \in \mathcal{Q}' $ according to the exponential mechanism with parameter $\nicefrac{\epsilon}{2T}$ and score function 
        \begin{equation*}
            s_{i}( \D_B, q)=n \left|q\left(\D_{i-1}\right)-q( \D_B )\right|. 
        \end{equation*} \; 
        \emph{Laplace Mechanism:} Let $m_i  = q_i  \left( \D_B \right) + \frac{1}{n} Lap \left( \nicefrac{2T}{\epsilon} \right) $ .\; 
        \emph{Multiplicative Update:} Update $\D_{i-1}$ using the rule 
        \begin{equation*}
            \tilde{\D}_i\left( x \right) \propto \D _{i-1}\left( x \right)   \exp\left(  \frac{ q_{i} \left( x \right)  \left(   m_i - q_i( \D_{i-1} ) \right)  }{2 }  \right) . 
        \end{equation*} \;
            \emph{KL Projection:} Project $\tilde{\D}_i$ onto the polytope $\mathcal{K} = \big\{ \vec z : z_i \geq 0, \; \sum\limits_{i=1}^N z_i =1 , z_i \leq \frac{1}{\sigma N} \big\} $ of smooth distributions:
            \begin{equation*}
                 \D_i =  \argmin_{ \D  \in \mathcal{K}}  \mathrm{D}_{\mathrm{KL}}( \D \|\tilde{\D }_i  )
             \end{equation*} 
        }
        Let $\overline{\D} = \frac{1}{T}  \sum_{i=1}^T \D_{i}   $.  \; 
        \KwOut{Distribution $\overline{\D}$. }   
    \end{algorithm}

\subsection{Proof of Theorem~\ref{thm:PSMWEM}}
As before, let $\overline{\D_B}$ be a corresponding $\sigma$-smooth distribution that matches the query answers for the $(\sigma, 0)$-smooth data set $B$.
Define $ {\Psi}_i =   \sum_{x \in  \X } \overline{\D_{B} } (x) \log \left( \overline{\D_{B} } (x)  / {\D}_{i}(x)\right) $ and  $ \tilde{\Psi}_i =   \sum_{x \in  \X }  \overline{\D_{B} } (x) \log \left( \overline{\D_{B} } (x)  / \tilde{\D}_{i}(x)\right) $ as the intermediate potential. 
    From \hyperref[lem:potential_improve]{Lemma \ref*{lem:potential_improve}}, we know   
        \begin{equation*}
            \Psi_{i-1} - \tilde{\Psi}_{i} \geq\left(\frac{q_{i}\left( \D_{i-1}\right)-q_{i}(\D_B)}{2 }\right)^{2}-\left(\frac{m_{i} - q_{i}( \D_B) }{2  }  \right)^{2}. 
        \end{equation*}
    Using the properties of relative entropy, we show the following claim.  
    \begin{claim}
        For every $i \leq T$, we have $\tilde{\Psi}_i \geq \Psi_i$.
    \end{claim}
    \begin{proof}
        The claim follows from the following fact about the KL divergence. 
        Let 
        \begin{equation*}
            \D_i =\argmin_{\D \in \mathcal{K}} \mathrm{D}_{\mathrm{KL}}( \D \| \tilde{\D }_{i}   )
        \end{equation*} 
        for some convex set $\mathcal{K}$.
        Then, for $ \overline{\D_{B} } \in \mathcal{K}$, 
        \begin{equation*}
            \mathrm{D}_{\mathrm{KL}}( \overline{\D_{B} } \| \tilde\D_i \, ) \geq \mathrm{D}_{\mathrm{KL}}\left( \overline{\D_{B} } \| \D_i \right)+\mathrm{D}_{\mathrm{KL}}\left( \D_i \| \tilde\D_i \right).
        \end{equation*} 
        The claim follows by $ \tilde{\Psi}_i = \mathrm{D}_{\mathrm{KL}}( \D_B \| \tilde\D_i\, )$, $ \Psi_i = \mathrm{D}_{\mathrm{KL}}\left(\D_B \| \D_i \right) $ and $\mathrm{D}_{\mathrm{KL}}\left( \D_i \| \tilde\D_i \right) \geq 0 $. 
    \end{proof}

    Together this gives 
        \begin{equation*}
            \Psi_{i-1} - {\Psi}_{i} \geq\left(\frac{q_{i}\left( \D_{i-1}\right)-q_{i}(\D_B)}{2 }\right)^{2}-\left(\frac{m_{i} - q_{i}( \D_B) }{2  }  \right)^{2}. 
        \end{equation*}
 The remainder of the analysis follows that of \autoref{thm:DPMW}.
    Note that we have $\overline{\D}$ is $\sigma$-smooth since each $ \mathcal{D}_i \in \mathcal{K } $ and $\mathcal{K}$ is a convex set.
   By \hyperref[lem:SmoothCover]{Lemma \ref*{lem:SmoothCover}}, we have $ \abs{ q' \left(\overline{\D} \right)- q\left( \overline{\D} \right) } \leq \nicefrac{2\gamma}{\sigma} $. Thus, 
    \begin{align*}
        \abs{  q\left( \D_B  \right) - q \left(\overline{\D} \right)  } & = \abs{  q\left( \D_B \right)  - q'\left( \D_B \right)    + q' \left( \D_B \right) -q'\left(\overline{\D} \right)   + q' \left(\overline{\D}\right) 
        - q\left(\overline{\D} \right) }  \\
        & \leq \abs{  q\left( \D_B \right)  - q'\left( \D_B \right) }   +\abs{  q' \left( \D_B \right) -q'\left(\overline{\D} \right) }  + \abs{ q' \left(\overline{\D} \right) 
        - q\left(\overline{\D} \right) } \\
        & \leq \frac{4 \gamma }{\sigma} + 2  \sqrt{\frac{\log \nicefrac{1}{\sigma}  }{T}}+\frac{10 T  d \log \left( \nicefrac{41}{ \gamma} \right)     }{\epsilon n } .
    \end{align*}
    Setting $ \gamma = \nicefrac{\sigma}{4n} $, we get 
    \begin{equation*}
        \abs{  q\left( \D_B  \right) - q \left( \overline{\D} \right)  } = \frac{1}{n} + 2  \sqrt{\frac{\log\left( \nicefrac{1}{\sigma} \right) }{T}}+\frac{10 T d \log\left( \nicefrac{4n}{\sigma} \right) }{\epsilon n }. 
    \end{equation*}
    
    \subsection{Running Time of \hyperref[alg:DataRelease]{Projected Smooth Multiplicative Weights Exponential Mechanism}}
    The running time is similar to the running time \hyperref[sec:running_SMWEM]{Smooth Multiplicative Weights Exponential Mechanism},with the additional projection step in each step. 
    Note that the projection in each step is a convex program and can be solved in time $\mathrm{poly} \left( \abs{\X} \right)$.
    This gives us a total running time of $O \left(n \left( 41 n / \sigma \right)^d + T   \left( 41 n / \sigma \right)^d \abs{\X} + T\mathrm{poly}(|\X|) \right).$

    In addition to the improvements discussed in the previous sections, the projection step can be performed faster by taking an approximate  Bregman projection as considered by \cite{ApproxBreg}. 
    Incorporating this into our algorithm would lead to significant speed ups. 

\section{Smooth Data Release using SmallDB Algorithm} \label{app:SmallDB}
    In this section,, we look at a different algorithm to get differential privacy when dealing with $\left( \sigma , \chi \right)$-smooth data sets. 
 {Our algorithm displayed below uses several pieces that have been introduced} by \cite{SmallDB} and \cite{DPMultWeights}.

\begin{algorithm}[H]
    \SetAlgoLined
    \caption{Subsampled Net Mechanism}
    \label{alg:SubNetMech}   
     \DontPrintSemicolon
    \KwIn{ Database $ B $ of size $n$, Query set $\mathcal{Q}$, Privacy parameter $\epsilon$, Subsampling parameter $M$, Accuracy parameter $\gamma$.}
    Sample (with replacement) a subset $V$ of size $M$ from $\X$. \;
    Sample $B'$ from amongst all data sets supported on $V$ of size 
    \begin{equation*}
        O\left( \frac{d   }{\gamma^2} \right) 
    \end{equation*}
    with probability proportional to 
    \begin{equation*}
        \exp \left( -  \frac{\epsilon \cdot n \cdot s\left( \D_{B'} , \D_B \right)  }{2}   \right)
    \end{equation*}
    where $s \left( \D_{B'} , \D_B \right)  = \max_{q \in \mathcal{Q}}  \abs{q(\D_B)  - q(\D_{B'}) }  $. \;
    \KwOut{Database $B'$}  
\end{algorithm}

First, we analyze the privacy of this algorithm.
\begin{theorem}
    The Subsampled Net Mechanism is $(\epsilon,0)$ differentially private. 
\end{theorem}
\begin{proof}
    The privacy claim follows from the privacy of the exponential mechanism. 
\end{proof}

 {Next we bound the error of this mechanism. Let us recall the standard uniform convergence bound.}

\begin{fact}[Uniform Convergence for VC Classes, see e.g. \cite{shalev2014understanding}] \label{fact:uniformconv}
    Let $\X$ be the domain, $\mathcal{Q}$ be a class of queries over $\X$ with VC dimension $d$ and let $\D$ be a distribution. 
    Let $\D'$ be a distribution gotten by sampling $ O\left(  (\log(2 / \eta  ) + d) / \gamma^2 \right) $ items iid from $\D$ and normalizing the frequencies. 
    Then, with probability $1 - \eta$, for all $q \in \mathcal{Q} $, $\abs{q(\D') -q(\D) } \leq \gamma $. 
 \end{fact}

In the following, we use the above fact to show that a randomly sampled subset of the universe approximates a $\left( \sigma , \chi \right)$-smooth database. 
The proof largely follows the domain reduction lemma of \cite{DPMultWeights}  {that achieve a similar bond by with a dependence on $\log(|\Q|)$}. We include this proof for completeness. 

\begin{lemma} \label{lem:subsampling}
    Let $\X $ be a data universe and $\mathcal{Q}$ a collection of queries over $ \X$ with VC dimension $d$  and $\D$ be $\left(  \sigma , \chi \right)$-smooth with respect to $\mathcal{Q}$. 
    Let $V \subset \X  $ of size $M$ be sampled from $\X$ at random with replacement with 
    \begin{equation*}
        M =    O\left(   \frac{  \log\left( 1 / \eta  \right) + d   }{ \sigma \gamma^2  } \right) .
    \end{equation*}
    Then, with probability $1-\eta$, there exists a $\D'$ on $V$ such that for all $q \in \mathcal{Q} $ 
    \begin{equation*}
        \abs{q\left( \D \right)  - q\left( \D' \right)  } \leq \chi + \gamma.
    \end{equation*}
 \end{lemma}
 \begin{proof}
     Let $\D_1$ be $\sigma$-smooth distribution that witnesses the $\left( \sigma, \chi \right)$-smoothness of $\D$. 
     If we could sample from $\D_1$, we would be done from \hyperref[fact:uniformconv]{Fact \ref*{fact:uniformconv}}. 
     But we want to get a subset that is oblivious to the distribution $\D$. 
    To achieve this, we use the smoothness of $\D_1$. 

    The idea is to sample from $\D_1$ using rejection sampling. 
    Since $\D_1$ is $\sigma$-smooth, the following procedure produces samples from $\D_1$: sample from the uniform distribution and accept sample $u$ with probability $\sigma N \D_1\left(u\right) $. 
     Note that accepted samples are distributed according to $\D_1$. 
     We repeat this process until $O\left(  (\log(2 / \eta  ) + d) / \gamma^2 \right)$ samples are accepted. 
     Since the accepted samples are distributed according to $\D_1$, from \hyperref[fact:uniformconv]{Fact \ref*{fact:uniformconv}}, there is a distribution  {$\D_2$} supported on the accepted samples such that  {with probability at least $1 - \eta/2  $    for all $q \in \mathcal{Q} $,}
     \begin{equation*}
         \abs{ q \left( \D_2 \right) - q\left( \D \right) } \leq \chi + \gamma.
     \end{equation*}

     Let $S_1$ be the coordinates corresponding to the accepted samples and $S_2$ be the coordinates corresponding to the rejected ones. 
     The key observation is that $S= S_1 \cup S_2$ is subset generated by sampling from the uniform distribution and has a distribution supported on it that approximates $\D$.
     So, it suffices to bound the size of $S$. 
     The probability that a given sample gets accepted is 
     \begin{equation*}
         \sum_{x \in \mathcal{X}} \frac{\D_1 \left( x \right) N \sigma }{N} = \sigma .
     \end{equation*}
     Thus the expected number of samples needed in the rejection sampling procedure is $M = O\left(  \frac{  \log\left( 2 / \eta  \right) + d   }{ \sigma \gamma^2  } \right) $. 
     Using a Chernoff bound, we can bound the probability that this is greater than its mean by a factor of $4$ by 
     \begin{equation*}
         e^{ - M } \leq \frac{\eta}{2}
     \end{equation*} 
     where we used that fact that $M \geq \log \left( 2/\eta \right)$ . 
  \end{proof}

 {We are finally ready to prove our theorem.}

\begin{theorem}
    For any  {data set $B$} that is $(\sigma,\chi)$-smooth with respect to a set of queries $\mathcal{Q}$ of VC dimension $d$, the output $\D''$ of the Subsampled Net Mechanism satisfies that 
    with probability $1- \eta$,    
    for all $q \in \mathcal{Q}$
    \begin{equation*}
        \abs{q\left(\, {\D_B}\right) - q\left( \D'' \right)  } \leq \chi + \tilde{O} \left( \sqrt[3]{ \frac{ d \log\left( 1 / \sigma \right) +\log\left( 1 / \eta \right)  }{\epsilon n  } } \right)  
    \end{equation*}

\end{theorem}
\begin{proof}
    Consider a subset $V$ sampled  with size $M = O\left(   \frac{  \log\left( 1 / \eta_1 \right) + d   }{ \sigma\gamma^2  }  \right) $ where $\eta_1$ and $\gamma$ are parameters we will set later. 
    From \hyperref[lem:subsampling]{Lemma \ref*{lem:subsampling}}, with probability $1-\eta_1$ we have that there exists a distribution $\D'$ supported on $V$ such that for all $q \in \mathcal{Q}$
    \begin{equation*}
        \abs{q(\D') - q(\D_B) } \leq \chi + \gamma. 
    \end{equation*}
    Let us work conditioned on this event. 
    Let $A$ denote the set of all data sets supported on $V$ and let $C$ denote all data sets supported on $V$ with size $ O\left(  d \gamma^{-2} \right) $. 
    From \hyperref[fact:uniformconv]{Fact \ref*{fact:uniformconv}}, for any distribution $\D_1$ supported on $V$, there is a data set in $C$ whose distribution $\D_2$ satisfies 
    \begin{equation*}
        \abs{ q\left( \D_1 \right) - q \left( \D_2 \right) } \leq \gamma. 
    \end{equation*}

    We recall the guarantees of the exponential mechanism (see e.g. \cite{AlgorithmicFoundationsDp}):  
        Let ${B''}$ be the data base output by the exponential mechanism. 
        Then, 
        \begin{equation*}
            \Pr \left[ s \left( \D_{B''} , \D_B \right) \geq \min_{B_1 \in C} s \left( \D_{B_1  } , \D_B \right) - \frac{2}{\epsilon n} \left( \log \abs{C} + t \right) \right] \leq e^{ - t}, 
        \end{equation*}
        where $s\left( \D_B , \D_{B'} \right)  {=}\max_{q \in \mathcal{Q}}  \abs{q(\D_B)  - q(\D_{B'}) }  $. 
    Note that $\log \abs{C} \leq M^{O\left(  d \gamma^{-2} \right) } $. 
    Thus, with probability $1- \eta_2$, 
    \begin{equation*}
        s \left( \D_{B''} , \D_B \right) \geq \min_{B_1 \in C } s \left( \D_{B_1  } , \D_B \right) - \gamma
    \end{equation*}
    for 
    \begin{equation*}
         \gamma \geq \frac{4  }{\epsilon n } \log  \frac{M^{O( d \gamma^{-2} ) }}{ \eta_2 }. 
    \end{equation*}
    Since, $\min_{B_1 \in C} s \left( \D_{B_1  } , \D_B \right) \leq \chi + 2 \gamma $,
    setting  $\eta_1 = \eta_2 = \eta/2$ and solving for $\gamma$, we get 
    \begin{equation*}
        \gamma = \tilde{O} \left( \sqrt[3]{ \frac{ d \log\left( 1 / \sigma \right) + \log\left( 1 / \eta \right)  }{\epsilon n  } } \right)
    \end{equation*}
    as required. 
\end{proof}

\subsection{Running Time of \hyperref[alg:SubNetMech]{Subsampled Net Mechanism}}

The running time of the algorithm involves first sampling $M$ elements uniformly from the domain which takes time $O\left(  M \log \abs{\X} \right) $. 
Each query needs to be evaluated on the data set $B$ which takes time $n \abs{ \mathcal{Q} } $. 
Evaluating and sampling from all data bases as required by the exponential mechanism naively takes time $M^{O\left(  d \gamma^{-2} \right) }$. 
As discussed earlier, this can be sped up using sampling for approximation.

\end{document}